\documentclass[11pt]{article}
\usepackage[final]{acl}
\usepackage{times}
\usepackage{latexsym}
\usepackage[T1]{fontenc}
\usepackage[utf8]{inputenc}
\usepackage{amsmath,amssymb,amsthm}
\usepackage{booktabs}
\usepackage{graphicx}
\usepackage{url}
\usepackage{xcolor}
\usepackage{microtype}
\usepackage{listings}
\usepackage{mathtools}
\usepackage{bm}
\usepackage{multirow}
\usepackage{colortbl}
\usepackage{tabularx}
\newtheorem{theorem}{Theorem}[section]
\newtheorem{proposition}[theorem]{Proposition}
\newtheorem{corollary}[theorem]{Corollary}

\newtheorem{remark}{Remark}[section]
\newtheorem{definition}{Definition}[section]
\newcommand{\Dtrain}{\mathcal{D}_{\mathrm{train}}}
\newcommand{\Dood}{\mathcal{D}_{\mathrm{ood}}}
\newcommand{\Forgetting}{\mathcal{F}}
\newcommand{\model}{\mathcal{M}_\theta}
\newcommand{\adapter}{\mathcal{A}_\phi}
\newcommand{\R}{\mathbb{R}}
\newcommand{\E}{\mathbb{E}}
\newcommand{\1}{\mathbf{1}}
\definecolor{bestgreen}{HTML}{E8F5E9}
\definecolor{worstred}{HTML}{FFEBEE}
\title{When to Adapt: Conditional Memory Adapters for Retention-Preserving Domain Specialization}
\author{
  Jiayu Hou \quad Lei Wang\thanks{Corresponding author} \\
  University of Electronic Science and Technology of China, Chengdu, China \\
  \texttt{202421230134@std.uestc.edu.cn} 
}
\begin{document}
\maketitle
\begingroup
\renewcommand{\thefootnote}{}
\footnotetext{Accepted to Findings of EMNLP 2026.}
\endgroup
\begin{abstract}
Large language models deployed in specialized domains must improve in-domain performance without sacrificing general capabilities. Existing parameter-efficient fine-tuning methods are typically always on: their learned perturbations are applied to every input, which can degrade out-of-domain (OOD) performance. We propose \textbf{Engram Adapter}, a framework that repurposes pretraining-time conditional memory as a post-hoc adapter for frozen LLMs. It uses multi-channel matching over local n-gram patterns with explicit occupancy tracking as a lightweight selectivity prior, making residual injection more likely on in-domain inputs while a learned scalar gate suppresses incoherent OOD retrievals. We evaluate on Qwen3-4B and Qwen3-8B with AG-News and MedMCQA as adaptation tasks and OOD benchmarks spanning reasoning, translation, code generation, and legal reasoning. Engram Adapter improves in-domain accuracy while preserving 99.4\%--100.1\% of average OOD performance; on LegalBench it slightly exceeds the frozen base model on average, whereas comparable always-on baselines degrade sharply. Mechanistic analyses show that although OOD activations are non-zero, gate and projection attenuation reduce residuals to approximately 0.08\% of hidden-state norm, yielding small KL drift and negligible accuracy change. These results suggest conditional activation is a promising route toward modular, retention-preserving domain specialization over frozen backbones.
\end{abstract}
\section{Introduction}
\label{sec:intro}
 
Recent years have seen growing demand for deploying large language models in specialized domains such as medicine, law, and finance. Under limited training and deployment budgets, parameter-efficient fine-tuning (PEFT) has become the dominant paradigm for post-training adaptation. Compared with full fine-tuning, PEFT freezes the backbone and updates only a small number of additional parameters, substantially reducing training cost and storage overhead while still delivering strong downstream performance. Representative methods include Adapter \citep{houlsby2019adapter}, LoRA \citep{hu2022lora}, DoRA \citep{liu2024dora}, and PiSSA \citep{meng2024pissa}.
 
\begin{figure*}[t]
    \centering
    \includegraphics[width=0.92\textwidth]{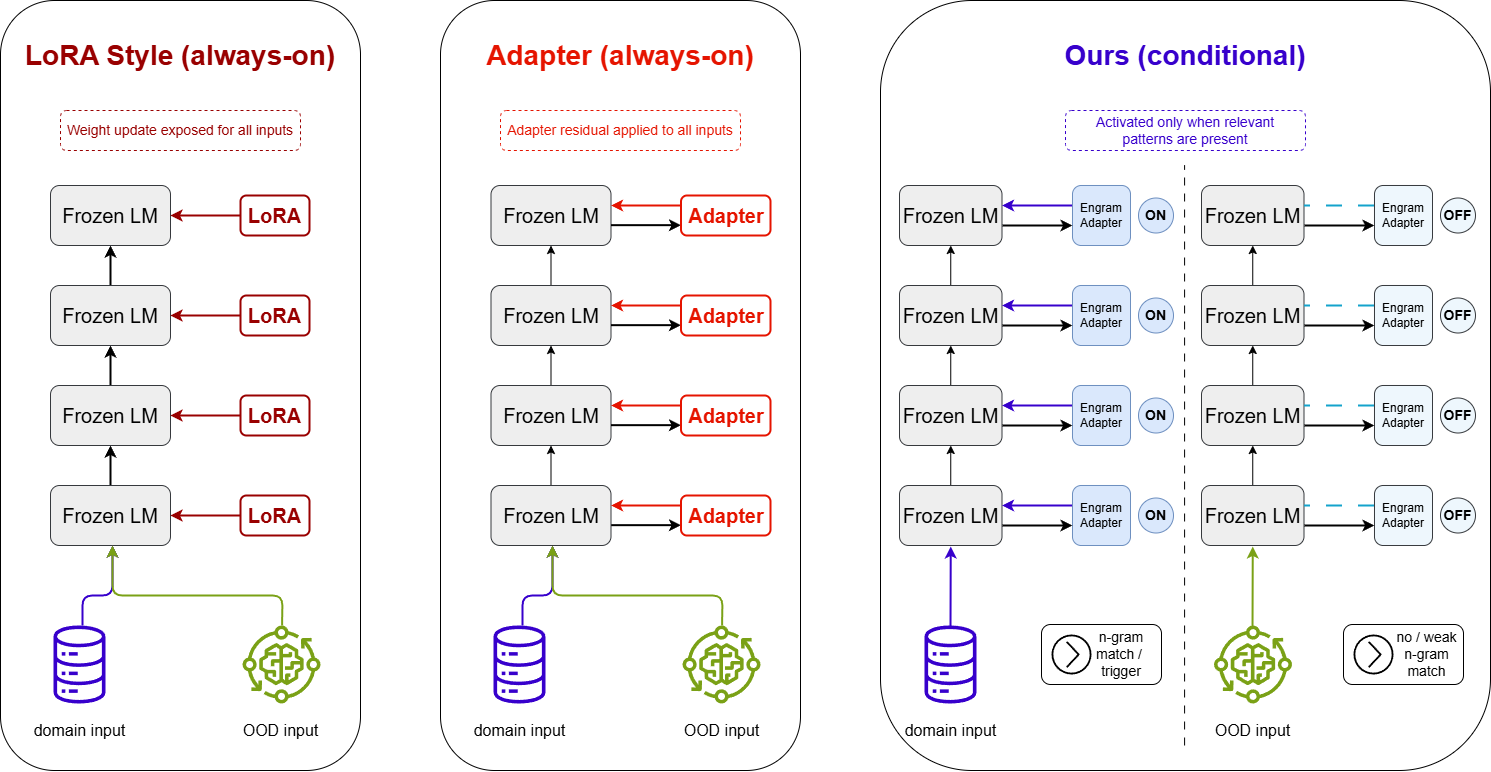}
    \caption{
    Conceptual comparison between always-on PEFT and conditional adaptation.
    LoRA-style methods expose learned low-rank weight updates to all inputs at
    inference time, while adapter-style methods expose an additional residual
    branch to all inputs. In contrast, our Engram Adapter conditionally exposes its
    residual through explicit local-pattern matching: domain-relevant inputs
    are more likely to activate the adapter, whereas out-of-domain inputs with
    no or weak matches receive masked or gate-attenuated residuals.
    }
    \label{fig:overall}
\end{figure*}
 
Despite their success, these methods often improve in-domain performance at the cost of out-of-domain general capabilities. Recent work has begun to revisit this adaptation--retention trade-off: \citet{lin2026sft} show that using a smaller learning rate can substantially mitigate general-performance loss while preserving comparable target-domain performance. However, learning-rate control mainly constrains the magnitude of parameter updates, rather than explicitly determining \emph{when} the adapted behavior should be invoked. Most PEFT methods are effectively always on: once trained, their residual modules or weight perturbations are applied to every input without distinguishing whether it belongs to the target domain. When domain data are small and distributionally narrow, such unconditional intervention can distort general representations and induce catastrophic forgetting. In other words, the central problem in domain specialization is not only \emph{how} to inject domain knowledge, but also \emph{when} that knowledge should be invoked. \textbf{This leads to the first challenge}: how can we make the adapter intervene when the input is relevant to the target domain, while remaining minimally intrusive otherwise? Figure~\ref{fig:overall} illustrates this distinction.
 
However, conditional activation requires a stable, lightweight signal of domain membership. Soft gating over continuous hidden representations alone is often insufficient in low-resource settings, where the boundary between in-domain and out-of-domain inputs is not cleanly separable and limited data make it difficult for a gate to learn a robust decision rule. Our motivation is directly inspired by DeepSeek Engram \citep{deepseek2026engram}, which shows that local $n$-gram patterns can serve as explicit triggers for conditional memory retrieval through efficient hash-based lookup. We use this mechanism as a \emph{selectivity prior}: the relatively stable terminology and collocations that characterize specialized corpora bias the adapter toward target-domain activation, but do not by themselves guarantee OOD transparency. \textbf{This raises the second challenge}: how can a conditional memory mechanism originally designed for pretraining be repurposed into a lightweight post-hoc adapter that provides sufficiently discriminative domain signals without collapsing into another form of always-on PEFT?
 
Moreover, a practically useful specialization method must keep overhead manageable, remain modular, and stay stable across model scales, task types, and capability dimensions. \textbf{This gives rise to the third challenge}: how can conditional activation be realized as a reusable PEFT mechanism with stable behavior across scales and evaluation domains, while retaining practical modularity for deployment?
 
Finally, conditional activation is not an ideal binary process. Different domains share vocabulary and local structural patterns, and any hash-based mechanism can produce non-zero false activations. A retention-preserving adapter must therefore explain not only why it activates more on in-domain inputs, but also why its perturbation remains small enough on out-of-domain inputs even when some activations occur. \textbf{This leads to the fourth challenge}: how can we maintain strong empirical preservation of general capabilities under unavoidable false activations, and provide a mechanistic explanation for this behavior?
 
To address these challenges, we propose Engram Adapter, a parameter-efficient framework that repurposes DeepSeek Engram from a pretraining-time conditional memory component into a post-hoc adapter for frozen large language models. Our framework uses an occupancy-tracked joint mask over multiple hash-indexed lookups to provide a lightweight selectivity prior, and combines a learned scalar gate with projection attenuation to suppress incoherent retrievals under false activations. The core evidence is mechanistic: we examine how activation selectivity relates to small residual perturbations, small output-distribution drift, and stable downstream behavior. To balance specialization and deployability, we inject the adapter into only a small set of key layers and share memory tables across layers, yielding a modular design compatible with maintaining separately trained adapters alongside a shared frozen backbone, although direct multi-domain adapter-switching evaluation is left to future work.
 
\paragraph{Contributions.} Corresponding to the four challenges above:
\begin{enumerate}
  \item We propose a conditional-activation framework that turns domain specialization from uniform parameter perturbation into selective knowledge injection driven by local pattern matching.
  \item We introduce adapter-oriented architectural modifications including backward-looking $n$-gram channels, occupancy-based joint masking, and learned scalar gating.
  \item We conduct systematic experiments on Qwen3-4B and Qwen3-8B with AG-News and MedMCQA as adaptation tasks and OOD evaluation spanning reasoning, translation, code generation, and legal reasoning, showing improved target-domain performance with more stable general-capability preservation.
  \item We provide mechanistic residual and output analyses revealing a perturbation cascade from domain selectivity to residual suppression to KL/output stability and finally to stable OOD behavior.
\end{enumerate}
\section{Related Work}
\label{sec:related}
\paragraph{PEFT and always-on adaptation.}
Parameter-efficient fine-tuning adapts frozen language models with a small set
of additional or reparameterized parameters. LoRA and its variants
\citep{hu2022lora,dettmers2023qlora,zhang2023adalora}, adapters
\citep{houlsby2019adapter,pfeiffer2021adapter}, prefix tuning
\citep{li2021prefix}, and $(IA)^3$ \citep{liu2022ia3} reduce training and
storage cost, but their learned perturbations are typically exposed to every
input at inference time. They address \emph{how} to adapt efficiently, whereas
we focus on \emph{when} adaptation should intervene.
\paragraph{Forgetting-aware adaptation.}
Catastrophic forgetting has traditionally been mitigated by constraining
parameter changes, allocating task-specific capacity, or replaying data
\citep{mccloskey1989catastrophic,kirkpatrick2017ewc,mallya2018packnet,
rolnick2019experience}. Recent domain-specific SFT work shows that smaller
learning rates and TALR (Token-Adaptive Loss Reweighting) can reduce
general-capability degradation by changing update magnitudes or token-level
loss weights \citep{lin2026sft}. These methods still produce a single adapted
model whose updates are exposed unconditionally. Engram Adapter instead freezes
the backbone and conditions residual exposure on explicit local-pattern
membership.
\paragraph{Conditional memory and membership-based activation.}
Memory-augmented models retrieve auxiliary information at inference time, from
nearest-neighbor datastores \citep{khandelwal2020knn} or external documents
\citep{guu2020realm,lewis2020rag}, usually as an additional content source.
DeepSeek Engram \citep{deepseek2026engram} instead uses efficient hash lookup
over local $n$-gram patterns as a conditional memory mechanism. We repurpose
this pretraining-time idea as a post-hoc adapter: lookup retrieves a
domain-specific residual and supplies the activation signal controlling whether
that residual is exposed. Our use of occupancy-based membership differs from
standalone set membership because the occupancy-tracked joint mask controls a
neural residual path. Since false activations cannot be eliminated in natural
language, membership filtering is combined with a learned gate and projection
attenuation.
\section{Engram Adapter}
\label{sec:arch}
\subsection{Conditional-Activation Adapter Framework}
\begin{figure*}[t]
\centering
\includegraphics[width=0.92\textwidth]{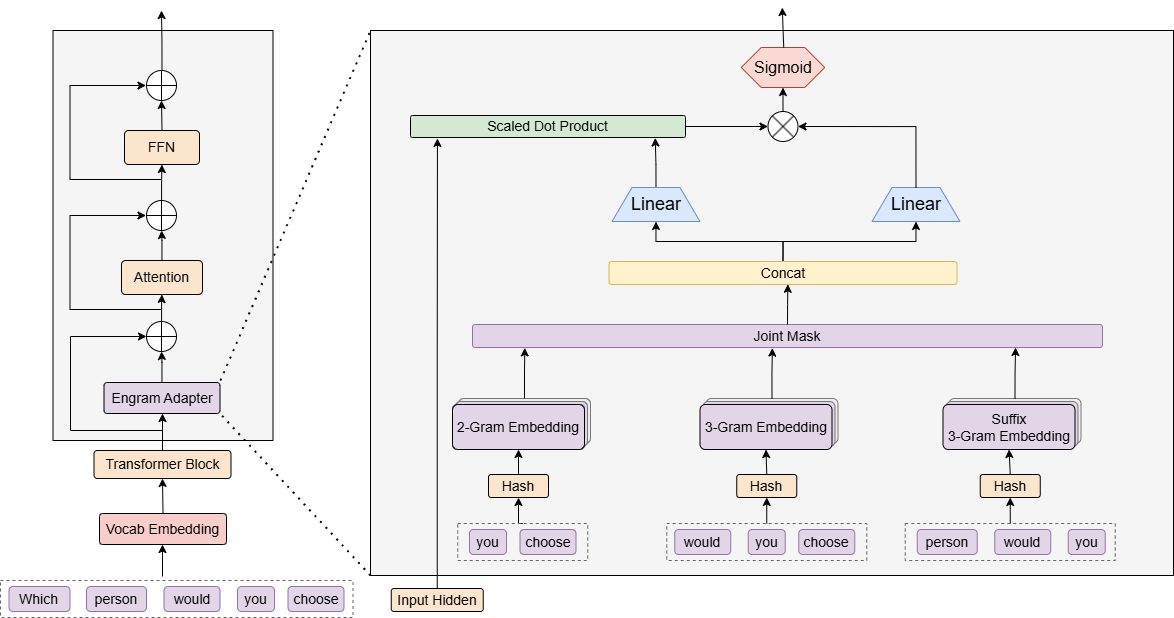}
\caption{Engram Adapter overview. Four self-attention layers
($\mathcal{L}=\{1,3,7,15\}$) receive hook-based adapters sharing hash memory
tables and occupancy masks. Input $n$-grams retrieve masked memory, which is
projected and gated before residual injection; only adapter parameters are
trainable.}
\label{fig:architecture}
\end{figure*}
Engram Adapter turns domain specialization into an input-conditioned residual
intervention. Given a frozen base model $\model$ (e.g., Qwen3-4B), we attach
one \texttt{EngramAdapter} module to the self-attention sublayer of each layer
in an exponentially spaced set $\mathcal{L} = \{1, 3, 7, 15\}$ via PyTorch
forward hooks, covering the model depth with only four injected layers. The
modified hidden state at layer $\ell \in \mathcal{L}$ is:
\begin{equation}
  u^{(\ell)}_{\mathrm{new}} = u^{(\ell)} + \alpha \cdot g^{(\ell)} \cdot
  \mathrm{mem}^{(\ell)}(x_{1:t}),
  \label{eq:residual}
\end{equation}
where $u^{(\ell)} \in \R^{B \times T \times D}$ is the hidden state, $\alpha =
0.2$ is a fixed scaling factor, $g^{(\ell)} \in [0,1]^{B \times T}$ is the
learned gate, and $\mathrm{mem}^{(\ell)}$ is the Engram memory retrieval. For
layers $\ell \notin \mathcal{L}$, the hidden state passes through unmodified.
All parameters $\theta$ of $\model$ satisfy $\nabla_\theta \mathcal{L} = 0$
throughout training. Unlike always-on PEFT modules, the residual in
\eqref{eq:residual} is exposed only through domain-matching memory and gate:
unmatched patterns are masked before projection, and incoherent retrievals are
attenuated by the gate.
\subsection{Domain-Matching Memory}
\label{sec:arch:tables}
To provide an explicit domain signal, each adapter reads from shared hash-indexed
memory tables indexed by local token patterns. For each suffix order
$n \in \mathcal{N} = \{2,3\}$, previous-context order
$n \in \mathcal{N}_{\mathrm{prev}} = \{3\}$, and hash index
$h \in \{0,\ldots,k-1\}$ with $k=4$, we maintain an embedding table and a
corresponding occupancy mask:
\[
  \begin{aligned}
  E_{n,h} &\in \R^{M \times d_h}, &
  F_{n,h} &\in \{0,1\}^{M}, \\
  M &= 90{,}007, &
  d_h &\in \{128,256\}.
  \end{aligned}
\]
We use $d_h=128$ for Qwen3-4B and $d_h=256$ for Qwen3-8B.
The 12 memory tables are shared across all injected layers,
amortizing memory cost and giving every layer the same domain-pattern inventory.
\subsection{Design Choices Beyond DeepSeek's Original Engram}
\label{sec:design-extensions}
We introduce five extensions relative to \citet{deepseek2026engram}.
\paragraph{Previous-context channels.}
The \texttt{prev\_ngram\_orders} setting addresses the collision issue noted by
\citet{deepseek2026engram}: hash collision probability grows rapidly as more
token IDs participate in the hash, making direct long-$n$-gram lookup
(e.g., $n \geq 4$) unreliable. We therefore use suffix orders $(2,3)$ rather
than direct long suffix channels, and introduce a separate set of hash tables that
indexes the $n$-gram immediately \emph{preceding} the current position. When a
suffix 3-gram $(x_{t-2}, x_{t-1}, x_t)$ and a previous-context 3-gram
$(x_{t-3}, x_{t-2}, x_{t-1})$ are used jointly, they cover a contiguous
4-token window $(x_{t-3}, \ldots, x_t)$ while keeping each hash over only three
token IDs.
\paragraph{Group-level joint masking.}
The original Engram uses per-bucket lookup without cross-hash coordination. We
add occupancy-tracked joint masking across all $k$ hashes within an order
group, providing the selectivity prior used by the retention mechanism
(Appendix~\ref{sec:theory}).
\paragraph{Learned scalar gate.}
The gate~\eqref{eq:gate} combines learned linear terms
$W_m \bm{c}(t)$ over the concatenated retrieval vector with a weighted
hidden--memory dot-product term
$(W_u \odot u(t))^\top \mathrm{mem}(t) / \sqrt{D}$. Ablation experiments
(Section~\ref{sec:activation}) show that the linear terms account for nearly
all of the gate's OOD suppression effect, while the dot-product term contributes
negligibly.
\paragraph{Random initialization with explicit occupancy mask.}
Unlike the original Engram's zero initialization, we initialize embedding
tables from $\mathcal{N}(0, 0.02^2)$ and maintain a separate boolean occupancy
mask $F_{n,h}$ that tracks which buckets have been accessed during training.
All occupancy flags are initialized to $F_{n,h}[i] = 0$. During training, each
accessed bucket is marked by setting $F_{n,h}[i] \leftarrow 1$; this
non-trainable occupancy buffer is saved with the adapter weights. At
inference time, the joint mask~\eqref{eq:joint_mask} is computed from $F_{n,h}$
rather than from $\|\cdot\|_1 > 0$. This decouples membership tracking from
embedding values, giving stronger initial gradients and avoiding false
negatives when an updated embedding remains numerically close to zero
(Proposition~\ref{prop:identity}).
\paragraph{Autoregressive support.}
A model-level forward pre-hook maintains a sliding window of
\texttt{history\_ids} aligned to the maximum $n$-gram order, enabling correct
$n$-gram construction during token-by-token generation.
\subsection{Forward Pass Details}
\label{sec:forward-pass}
\paragraph{Step 1: $n$-gram construction.}
For each position $t$ and suffix order $n \in \mathcal{N}$, the suffix $n$-gram
is $\bm{q}^{(n)}_t = (x_{t-n+1}, \ldots, x_t)$. For previous-context order $n
\in \mathcal{N}_{\mathrm{prev}}$, the preceding $n$-gram is
$\bm{p}^{(n)}_t = (x_{t-n}, \ldots, x_{t-1})$.
\paragraph{Step 2: Hash computation.}
For suffix $n$-grams, the bucket index for hash function $h$ is:
\begin{equation}
  \begin{aligned}
  \mathrm{idx}_{n,h}(t)
  &= \Bigl(
      \sum_{j=0}^{n-1} x_{t-j}
      (1{,}000{,}003 + 97j) \\
  &\qquad + s_{n,h}
    \Bigr) \bmod M,
  \end{aligned}
  \label{eq:hash}
\end{equation}
where $s_{n,h} = 10{,}000n + 97h$ is a deterministic seed. Previous-context
$n$-grams use the same formula with a seed offset of $1{,}000{,}000$.
\paragraph{Step 3: Lookup and occupancy-tracked joint masking.}
The retrieved chunk is $v_{n,h}(t) = E_{n,h}[\mathrm{idx}_{n,h}(t)] \in
\R^{d_h}$. The per-hash activation mask is determined by the occupancy flag
rather than the embedding norm:
\[
  \mu_{n,h}(t) = F_{n,h}\bigl[\mathrm{idx}_{n,h}(t)\bigr].
\]
The \emph{joint mask} for order group $n$ requires \emph{all} $k$ hashes to
retrieve filled buckets:
\begin{equation}
  m_n(t) = \prod_{h=0}^{k-1} \mu_{n,h}(t),
  \label{eq:joint_mask}
\end{equation}
and the filtered chunk is $\tilde{v}_{n,h}(t) = m_n(t) \cdot v_{n,h}(t)$.
If any single hash maps to an unfilled bucket, all $k$ retrieved embeddings
for that order group are zeroed out, regardless of the other hashes' status.
\paragraph{Step 4: Projection and gating.}
All filtered chunks are concatenated and projected:
\begin{equation}
  \begin{aligned}
  \bm{c}(t) &= \bigl[\tilde{v}_{n,h}(t)\bigr]_{n,h}, \\
  \mathrm{mem}(t) &= W_{\mathrm{proj}}\,\bm{c}(t).
  \end{aligned}
  \label{eq:proj}
\end{equation}
Here $\bm{c}(t) \in \R^{kN_gd_h}$ with
$N_g=|\mathcal{N}|+|\mathcal{N}_{\mathrm{prev}}|=3$,
$\mathrm{mem}(t) \in \R^D$, and
$W_{\mathrm{proj}} \in \R^{D \times kN_gd_h}$ is initialized to zero. The gate is
computed as:
\begin{equation}
  g(t) = \sigma\!\left(
     W_m\,\bm{c}(t)
    + \frac{(W_u \odot u(t))^\top \mathrm{mem}(t)}{\sqrt{D}}
  \right),
  \label{eq:gate}
\end{equation}
with $W_u \in \R^{D}$ and $W_m \in \R^{kN_gd_h}$ for the scalar gate variant.
The residual injection
follows \eqref{eq:residual}.
\paragraph{Trainable parameters.}
The trainable parameters are the shared memory tables plus per-layer projection
and gate parameters, totaling approximately 154M parameters in our Qwen3-4B
setting. This is larger than LoRA $r$=64 ($\approx$132M) and LoRA $r$=32
($\approx$66M), but many memory rows remain inactive under sparse occupancy, and
adapters are injected into only 4 of 36 layers. Appendix~\ref{app:param-count}
gives the full calculation.
\section{Experiments}
\label{sec:experiments}
We evaluate the Engram Adapter against LoRA baselines on Qwen3-4B and Qwen3-8B,
focusing on two axes: (1) domain adaptation effectiveness, and (2) preservation
of general capabilities (retention preservation). We conduct experiments across
two model scales, two domain adaptation tasks, and 12 legal reasoning tasks to
assess generality.
\subsection{Experimental Setup}
We evaluate Qwen3-4B and Qwen3-8B \citep{qwen2025qwen3} on limited-data
domain adaptation and OOD retention. Qwen3-4B is trained on 1{,}000 examples
from AG-News or MedMCQA and evaluated on ARC-Challenge, the opposite domain,
FLORES, and MBPP; Qwen3-8B is trained on 10{,}000 AG-News examples and
evaluated on MedMCQA, ARC-Challenge, and LegalBench. LoRA $r\in\{32,64\}$ and
Engram Adapter are run with three seeds (42, 7783, 114514) on 4B, while
non-LoRA non-Engram 4B baselines and all reported 8B baselines use seed 7783.
For Qwen3-4B, all non-MBPP evaluations use zero-shot prompting, while MBPP
uses 3-shot prompting; all Qwen3-8B evaluations use zero-shot prompting.
Engram Adapter uses layers
$\mathcal{L}=\{1,3,7,15\}$, suffix orders $(2,3)$, previous-context order
$(3,)$, $k=4$, table size $M=90{,}007$, and per-hash dimension
$d_h=128$ on Qwen3-4B and $d_h=256$ on Qwen3-8B. Appendix~\ref{app:exp-setup}
provides the full setup summary and benchmark details.
\subsection{Experiment 1: AG-News Domain Adaptation (4B)}
Table~\ref{tab:main_results} presents the complete results for the AG-News
domain adaptation experiment on Qwen3-4B.
\begin{table*}[t]
\centering
\caption{AG-News domain adaptation results on Qwen3-4B. LoRA and Engram Adapter
results are reported as mean $\pm$ std across three random seeds
(42, 7783, 114514).
For Qwen3-4B, AG-News, ARC-C, MedMCQA, and FLORES are evaluated zero-shot;
MBPP uses 3-shot prompting.
\textbf{Bold}: best among adapted models; \underline{underline}: second best.
AG-News is the in-domain target task; the remaining four benchmarks
measure out-of-domain retention.}
\label{tab:main_results}
\small
\begin{tabular}{@{}l c c c cc c@{}}
\toprule
& \textbf{In-Domain} &
  \multicolumn{5}{c}{\textbf{Out-of-Domain (General Capabilities)}} \\
\cmidrule(lr){2-2} \cmidrule(lr){3-7}
\textbf{Method}
  & AG-News & ARC-C & MedMCQA & FLORES & FLORES & MBPP \\
  & Acc.\% & Acc.\% & Acc.\% & BLEU & chrF & pass@1 \\
\midrule
Base Model
  & 85.14 & 88.99 & 53.65 & 29.85 & 58.85 & 66.80 \\
\midrule
LoRA $r\!=\!64$
  & 89.89{\tiny$\pm$0.50} & \underline{88.14{\tiny$\pm$0.49}} & 53.72{\tiny$\pm$0.33}
  & 28.96{\tiny$\pm$0.44} & 58.13{\tiny$\pm$0.18} & 56.67{\tiny$\pm$12.21} \\
LoRA $r\!=\!32$
  & 90.15{\tiny$\pm$0.14} & 88.11{\tiny$\pm$0.38} & \textbf{54.90{\tiny$\pm$0.27}}
  & 29.55{\tiny$\pm$0.19} & 58.61{\tiny$\pm$0.06} & \underline{64.93{\tiny$\pm$0.84}} \\
Adapter
  & \textbf{90.54} & 87.88 & 53.41 & \underline{29.91} & 58.56 & 56.00 \\
DoRA $r\!=\!32$
  & \underline{90.30} & \underline{88.14} & \underline{54.58} & \textbf{30.08} & \underline{58.75} & 59.80 \\
PiSSA $r\!=\!32$
  & 88.79 & 19.88 & 32.23 & 0.00 & 1.87 & 0.00 \\
$(IA)^3$
  & 86.74 & 86.01 & 53.26 & 21.22 & 55.74 & 56.40 \\
\midrule
\textbf{Engram Adapter}
  & 87.93{\tiny$\pm$0.64} & \textbf{89.45{\tiny$\pm$0.47}} & 53.57{\tiny$\pm$0.20} & 29.68{\tiny$\pm$0.33} & \textbf{58.82{\tiny$\pm$0.09}} & \textbf{65.27{\tiny$\pm$0.50}} \\
\bottomrule
\end{tabular}
\end{table*}
\subsection{Experiment 2: MedMCQA Domain Adaptation (4B)}
To validate the generality of our findings, we repeat the experiment with a
medical domain task. Table~\ref{tab:medmcqa_results} presents the results.
\begin{table*}[t]
\centering
\caption{MedMCQA domain adaptation results on Qwen3-4B. LoRA and Engram Adapter
results are reported as mean $\pm$ std across three random seeds
(42, 7783, 114514).
For Qwen3-4B, MedMCQA, ARC-C, AG-News, and FLORES are evaluated zero-shot;
MBPP uses 3-shot prompting.
\textbf{Bold}: best among adapted models; \underline{underline}: second best.
MedMCQA is the in-domain target task; the remaining four benchmarks
measure out-of-domain retention.}
\label{tab:medmcqa_results}
\small
\begin{tabular}{@{}l c c c cc c@{}}
\toprule
& \textbf{In-Domain} &
  \multicolumn{5}{c}{\textbf{Out-of-Domain (General Capabilities)}} \\
\cmidrule(lr){2-2} \cmidrule(lr){3-7}
\textbf{Method}
  & MedMCQA & ARC-C & AG-News & FLORES & FLORES & MBPP \\
  & Acc.\% & Acc.\% & Acc.\% & BLEU & chrF & pass@1 \\
\midrule
Base Model
  & 53.65 & 88.99 & 85.14 & 29.85 & 58.85 & 66.80 \\
\midrule
LoRA $r\!=\!64$
  & 54.38{\tiny$\pm$0.33} & 88.22{\tiny$\pm$0.19} & \underline{86.30{\tiny$\pm$0.24}}
  & \underline{30.06{\tiny$\pm$0.29}} & 58.56{\tiny$\pm$0.18} & 64.07{\tiny$\pm$0.34} \\
LoRA $r\!=\!32$
  & \underline{57.10{\tiny$\pm$0.27}} & 88.14{\tiny$\pm$0.14} & 86.13{\tiny$\pm$0.12}
  & \textbf{30.17{\tiny$\pm$0.10}} & \textbf{58.97{\tiny$\pm$0.01}} & \underline{65.13{\tiny$\pm$0.09}} \\
Adapter
  & 55.49 & 87.97 & 84.61 & 29.23 & 58.82 & 56.40 \\
DoRA $r\!=\!32$
  & \textbf{57.28} & \underline{88.31} & 84.76 & 29.92 & \underline{58.86} & 57.00 \\
PiSSA $r\!=\!32$
  & 55.63 & 27.30 & 0.00 & 0.00 & 0.24 & 0.00 \\
$(IA)^3$
  & 54.91 & 62.88 & 84.58 & 20.93 & 55.63 & 58.00 \\
\midrule
\textbf{Engram Adapter}
  & 55.79{\tiny$\pm$0.74} & \textbf{88.48{\tiny$\pm$0.34}} & \textbf{86.31{\tiny$\pm$0.53}} & 29.89{\tiny$\pm$0.06} & 58.73{\tiny$\pm$0.15} & \textbf{66.37{\tiny$\pm$0.25}} \\
\bottomrule
\end{tabular}
\end{table*}
\section{Mechanistic Evidence for Retention}
\label{sec:activation}
The retention claim in this paper rests on empirical mechanism analysis rather
than on a standalone proof. The occupancy-tracked mask provides a selectivity prior: it
makes activation substantially more likely on target-domain local patterns, but
natural language overlap and correlated hash inputs still produce false
activations. We therefore test the complete pathway from activation, to
residual size, to output drift, to downstream behavior.
\paragraph{OOD activation is non-negligible.}
For each trained adapter (AG-News and MedMCQA), we compute the token-level
joint activation rate, exact $n$-gram overlap with the training set, and
empirical false positive rate across the three retained $n$-gram channels
(n2, n3, p3). The reduced channel set still produces non-zero OOD activation:
for the two highest-activation cross-domain pairs analyzed below, roughly
26--28\% of token positions activate at least one order group. This means the
Engram Adapter is \emph{not} structurally transparent on OOD inputs in the
strong sense assumed by the idealized analysis.
\paragraph{The independence assumption does not hold.}
Empirical false positives can exceed the idealized $\rho^k$ reference because
the $k$ hash functions share the same token-ID input space. Natural language
$n$-grams are drawn from a highly non-uniform distribution, causing hash
collisions to be correlated across hash functions rather than independent.
\paragraph{Gate attenuation as a residual-level defense.}
Given that OOD activation is non-negligible, the strong OOD retention
($\geq$99.4\%) appears to come primarily from the learned scalar gate and
incoherent projection cancellation (Section~\ref{sec:gate}), rather than from
activation sparsity alone. The occupancy mask provides a \emph{relative}
selectivity prior, while the gate provides residual-level suppression of
incoherent contributions.
To test this mechanistic hypothesis, we conduct a detailed residual and
output analysis on the two highest-activation OOD pairs: AG adapter $\to$
MedMCQA and MedMCQA adapter $\to$ AG-News.
\subsection{Residual Perturbation Analysis}
For each token at each injected layer, we measure the ratio $\|\Delta
u\|_2 / \|u\|_2$, where $\Delta u = \alpha \cdot g \cdot \mathrm{mem}$ is
the Engram residual and $u$ is the base hidden state.
Table~\ref{tab:residual} reports the results.
\begin{table}[h]
\centering
\caption{OOD residual perturbation analysis for the two highest-activation
cross-domain pairs (seed 7783). All statistics are computed across all tokens and
all four injected layers ($\mathcal{L} = \{1,3,7,15\}$). ``Active'' rows
condition on tokens where the joint mask fired ($m_n = 1$ for at least one
order group).}
\label{tab:residual}
\footnotesize
\setlength{\tabcolsep}{2pt}
\begin{tabularx}{\linewidth}{@{}>{\raggedright\arraybackslash}Xcccc@{}}
\toprule
& \multicolumn{2}{c}{\textbf{AG$\to$Med}} &
  \multicolumn{2}{c}{\textbf{Med$\to$AG}} \\
\cmidrule(lr){2-3} \cmidrule(lr){4-5}
\textbf{Metric} & All & Active & All & Active \\
\midrule
Mean act.               & \multicolumn{2}{c}{28.26\%}
                        & \multicolumn{2}{c}{26.28\%} \\
Mean gate               & 0.599 & 0.640 & 0.656 & 0.690 \\
Mean $\|\Delta u\|/\|u\|$ & 0.078\% & 0.275\% & 0.076\% & 0.290\% \\
P95 $\|\Delta u\|/\|u\|$  & 0.371\% & 1.124\% & 0.348\% & 1.233\% \\
\bottomrule
\end{tabularx}
\end{table}
Despite mean joint activation rates of 26--28\%, the mean residual ratio is
only $\approx$0.08\% of the hidden-state norm. Even conditioning on activated
tokens, the 95th-percentile perturbation is only $\approx$1.1--1.2\% of the
hidden-state norm. This supports the view that the gate and incoherent projection
jointly suppress the adapter's contribution to a negligible level.
\subsection{Output Distribution Analysis}
To assess whether these small residual perturbations propagate to model
outputs, we compare the adapted model's predictions against the base model
on the same OOD evaluation sets. Table~\ref{tab:kl_flip} reports the
KL divergence between output distributions and the rate of prediction
changes.
\begin{table}[h]
\centering
\caption{Output distribution analysis for OOD cross-domain pairs (seed 7783).
Prediction change rate measures the fraction of examples where the
argmax prediction differs between the base and adapted models.
Negative flips are examples that the base model got right but the
adapted model got wrong; positive flips are the reverse.}
\label{tab:kl_flip}
\footnotesize
\begin{tabular}{@{}l cc@{}}
\toprule
\textbf{Metric} & \textbf{AG$\to$Med} & \textbf{Med$\to$AG} \\
\midrule
Base accuracy           & 53.65\% & 85.14\% \\
Adapted accuracy        & 53.60\% & 86.57\% \\
Prediction changed      & 7.55\%  & 2.16\% \\
Negative flip rate      & 3.80\%  & 0.36\% \\
Positive flip rate      & 3.75\%  & 1.79\% \\
Mean KL to base         & 0.088   & 0.040 \\
P95 KL to base          & 0.406   & 0.091 \\
\bottomrule
\end{tabular}
\end{table}
For AG $\to$ MedMCQA, the prediction change rate is 7.55\%, with negative
and positive flips nearly balanced (3.80\% negative vs.\ 3.75\% positive),
yielding a negligible accuracy change from 53.65\% to 53.60\%. For Med
$\to$ AG-News, the perturbation is smaller: only 2.16\% of predictions
change, with positive flips exceeding negative flips (1.79\% vs.\ 0.36\%),
so accuracy improves from 85.14\% to 86.57\%.
\begin{table*}[!t]
\centering
\caption{AG-News domain adaptation results on Qwen3-8B. All methods are trained on
10,000 AG-News examples with seed 7783. AG-News is the in-domain target task;
MedMCQA, ARC-Challenge, and LegalBench measure out-of-domain retention. All reported
8B baselines use seed 7783, and all Qwen3-8B evaluations use zero-shot
prompting. OOD Avg.\ Ret. averages relative retention over MedMCQA,
ARC-Challenge, and LegalBench.}
\label{tab:8b_results}
\small
\begin{tabular}{@{}l cccc c@{}}
\toprule
& \textbf{In-Domain} & \multicolumn{3}{c}{\textbf{OOD}} & \textbf{OOD Avg.} \\
\cmidrule(lr){2-2} \cmidrule(lr){3-5} \cmidrule(lr){6-6}
\textbf{Method}
  & AG-News & MedMCQA & ARC-C & LegalBench & Ret.\% \\
  & Acc.\% & Acc.\% & Acc.\% & Avg.\ Acc.\% & \\
\midrule
Base Model
  & 79.49 & 55.51 & 89.77 & 71.91 & 100.00 \\
\midrule
LoRA $r\!=\!64$
  & \textbf{90.29} & \textbf{57.92} & \underline{90.53} & 44.67 & 89.1 \\
LoRA $r\!=\!64$ (lr=$10^{-5}$)
  & \underline{86.96} & \underline{57.33} & 89.68 & 61.08 & 96.0 \\
LoRA $r\!=\!64$ (lr=$10^{-6}$)
  & 82.18 & 53.60 & 89.52 & \underline{68.54} & 97.2 \\
TALR $r\!=\!64$
  & 79.57 & 56.87 & \textbf{91.04} & 65.16 & \underline{98.2} \\
\midrule
\textbf{Engram Adapter}
  & 86.47 & 56.83 & 89.52 & \textbf{72.53} & \textbf{101.0} \\
\bottomrule
\end{tabular}
\end{table*}
\subsection{Mechanistic Summary}
These experiments support the following perturbation pathway for OOD preservation:
\begin{enumerate}
  \item \textbf{Activation}: OOD tokens do trigger joint activation
    ($\approx$26--28\% of tokens), so retention cannot be explained by exact
    OOD transparency.
  \item \textbf{Residual}: Despite activation, the gate and incoherent
    projection suppress the injected residual to $<$0.08\% of the
    hidden-state norm on average ($<$1.2\% at P95 among activated tokens).
  \item \textbf{Output}: The small residuals produce limited output
    distribution drift (mean KL 0.04--0.09), with prediction changes
    affecting only 2--8\% of examples.
  \item \textbf{Behavior}: Net accuracy change is negligible
    ($-$0.05 to $+$1.43 percentage points), consistent with no catastrophic
    forgetting even when activation is non-negligible.
\end{enumerate}
\section{Scaling and Baseline Ablations on Qwen3-8B}
\label{sec:scaling}
To assess whether the Engram Adapter's properties hold at larger scale and to
compare against training-time retention controls, we conduct experiments on
Qwen3-8B with additional baselines---\textbf{LoRA with reduced learning rates}
($10^{-5}$ and $10^{-6}$) and \textbf{TALR} \citep{lin2026sft}---and evaluate on 12
LegalBench tasks \citep{guha2023legalbench}.
\paragraph{Additional 8B baselines.}
We include LoRA $r$=64 with lr=$10^{-5}$, LoRA $r$=64 with lr=$10^{-6}$,
and TALR LoRA $r$=64 \citep{lin2026sft} to test whether update-magnitude
reduction or token-adaptive loss reweighting can preserve OOD capabilities
without occupancy-conditioned activation. The lr=$10^{-5}$ run serves as a
near-matched in-domain LoRA comparison to Engram Adapter. Implementation
details for these baselines are provided in Appendix~\ref{app:8b-baselines}.
\paragraph{8B AG-News results.}
Table~\ref{tab:8b_results} presents the results on Qwen3-8B trained on AG-News.
The Engram Adapter achieves 86.47\% in-domain accuracy,
a larger gain than on 4B, while retaining OOD performance: MedMCQA 56.83\%,
ARC-Challenge 89.52\%, and LegalBench 72.53\% average accuracy. LoRA $r$=64 reaches
higher in-domain accuracy (90.29\%) but falls sharply on LegalBench (44.67\%).
The intermediate-learning-rate LoRA baseline closely matches Engram Adapter
in-domain (86.96\% vs.\ 86.47\%) but still drops to 61.08\% on LegalBench and
96.0\% average OOD retention. The remaining training-time baselines expose
complementary failure modes: the more conservative LoRA lr=$10^{-6}$ still
remains below the base model on MedMCQA, while TALR preserves OOD scores better but achieves only
marginal in-domain improvement (79.57\% vs.\ base 79.49\%).
\begin{figure}[t]
\centering
\includegraphics[width=\linewidth]{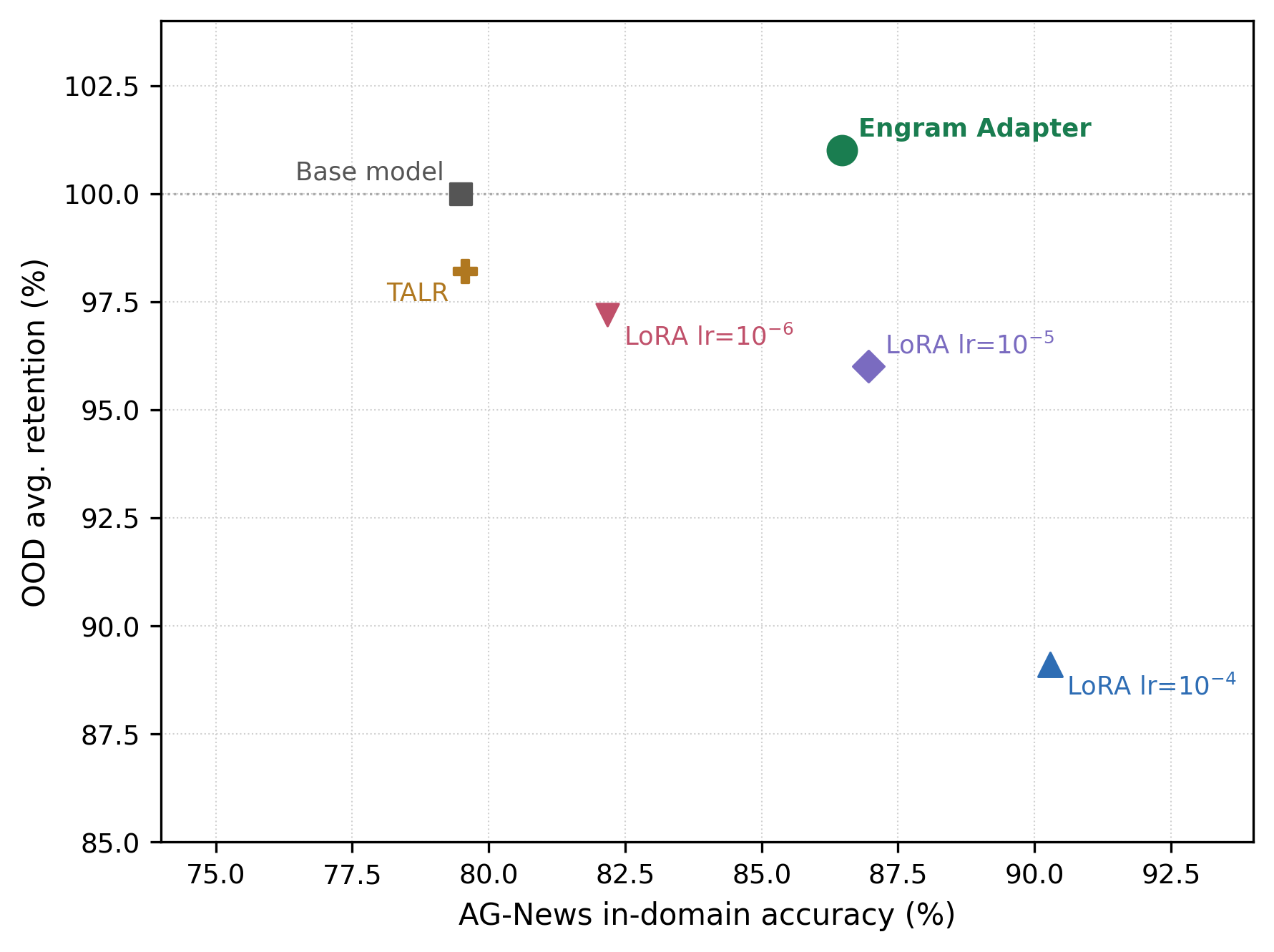}
\caption{Adaptation--retention trade-off on Qwen3-8B. Each point
represents a method trained on AG-News; the $x$-axis is in-domain
accuracy and the $y$-axis is average OOD retention over MedMCQA,
ARC-Challenge, and LegalBench (100\% = base model). The Engram Adapter is the only adapted method that combines a substantial
in-domain improvement with $\geq$100\% average OOD retention. LoRA variants
achieve higher in-domain accuracy in some settings but lower OOD retention,
while TALR preserves more OOD performance at the cost of little in-domain
improvement. All LoRA variants are $r$=64.}
\label{fig:pareto}
\end{figure}
\paragraph{Training-time controls are insufficient.}
These training-time baselines show that controlling update magnitude
or loss weights does not replace conditional activation. LoRA lr=$10^{-5}$
achieves near-matched in-domain performance but substantially worse LegalBench
retention than Engram Adapter; LoRA lr=$10^{-6}$ learns only modestly in-domain
and remains below the base model on MedMCQA (53.60\% vs.\ 55.51\%); TALR achieves strong OOD retention but
barely improves AG-News (79.57\% vs.\ base 79.49\%).
\section{Conclusion}
\label{sec:conclusion}
We presented the Engram Adapter, a parameter-efficient domain specialization
framework that repurposes DeepSeek's Engram conditional memory module as a
post-hoc adapter for frozen large language models. By combining hash-indexed
local-pattern matching, explicit occupancy tracking, and a learned scalar gate,
the adapter turns domain specialization from an always-on perturbation into a
conditional residual intervention guided by a selectivity prior. Experiments on
Qwen3-4B and Qwen3-8B show that this design improves target-domain performance
while preserving general capabilities across reasoning, translation, code
generation, and legal reasoning benchmarks. Mechanistic analyses provide the
central evidence: OOD inputs can activate the adapter, but the gate and
projection suppress the resulting residuals enough to keep KL drift, prediction
changes, and downstream accuracy stable. These results suggest that
conditional memory adapters are a promising route toward modular domain
specialization over frozen backbones. Direct multi-domain isolation and
composition experiments remain future work.
\section*{Limitations}\label{sec:limitations}
The current experiments span two model scales (Qwen3-4B and Qwen3-8B) and
three evaluation domains (news, medicine, law), but empirical validation on
even larger models (e.g., 27B parameters) and more challenging domain corpora
would further strengthen the findings. Although the shared memory tables keep
the Qwen3-4B Engram Adapter's total parameter count ($\approx$154M) within a
manageable range, it is roughly 1.2$\times$ the size of LoRA $r$=64
($\approx$132M). The
in-domain performance gap relative to LoRA (2--4 percentage points across
experiments) warrants investigation into whether larger training sets,
alternative $n$-gram orders, or tuned $\alpha$ values can close this gap
without sacrificing the observed retention behavior.
\paragraph{Inference overhead.}
Under fixed-length generation on MBPP (128 tokens), the Engram Adapter
incurs $\approx$22\% per-token overhead relative to the frozen base model
on Qwen3-4B. Forward-pass (prefill) overhead on OOD evaluation sets ranges from
$+1.45\%$ to $+39.9\%$. The measured overhead is not monotonic in the
activation rate, because it also depends on input length, batching,
Python-level hash lookup, and memory-access patterns. The current
implementation uses unoptimized Python-level hash lookup; fused GPU kernels
would likely reduce this overhead. Full efficiency and activation profiles
are provided in Appendix~\ref{app:efficiency}.
\paragraph{Artifacts and intended use.}
We use publicly available models, datasets, and benchmarks only for research
training and evaluation, and we cite their creators in the main text and
references. We do not redistribute third-party datasets or base model weights;
any released code or adapters should require users to obtain the underlying
artifacts under their original licenses and terms. Our use is intended to be
consistent with the research and evaluation purposes of these artifacts, and
the method is not presented as a deployable system for high-stakes legal or
medical decision making without additional validation.
\paragraph{Data privacy and content.}
We do not collect new human-subject data, and we rely on the published
documentation of the existing public datasets and benchmarks to understand
their domains and intended use. Some artifacts may include public names or
sensitive domain content, such as news, medical questions, or legal scenarios;
we therefore use only the released task fields, report aggregate metrics, and
do not redistribute raw examples or attempt to identify individuals. We do not
perform additional anonymization beyond the original artifact releases.

\appendix
\section{Additional Architecture and Implementation Details}\label{app:full-architecture}
\subsection{Trainable Parameter Count}
\label{app:param-count}
The trainable parameters consist of the shared memory tables plus per-layer
projection and gate parameters. The shared embedding tables contain:
\[
  \begin{aligned}
  &(|\mathcal{N}| + |\mathcal{N}_{\mathrm{prev}}|)
    \times k \times M \times d_h \\
  &\quad = 3 \times 4 \times 90{,}007 \times 128
    \approx 138\text{M}
  \end{aligned}
\]
parameters in total. Each of the $|\mathcal{L}|=4$ injected layers additionally
maintains a projection $W_{\mathrm{proj}} \in \R^{D \times (k \cdot N_g \cdot
d_h)}$ plus small scalar-gate vectors, contributing $\approx$16M
across all four layers. The total trainable parameter count is thus
$\approx$154M in the Qwen3-4B setting. For comparison, LoRA $r$=64 introduces $\approx$132M parameters
and LoRA $r$=32 introduces $\approx$66M, so the Engram Adapter is roughly
1.2$\times$ the size of the highest-capacity LoRA baseline. However, the memory
tables are highly sparse after training on a small corpus: most of the
1{,}080{,}084 embedding rows across all 12 tables remain at their random
initialization and are masked out by the occupancy flags, so the effective
active parameter count is substantially lower. Moreover, by injecting at only
4 out of 36 layers, the per-layer overhead is reduced by $\approx$89\% compared
to a hypothetical all-layer design.
\section{Implementation and Reproducibility Details}
\label{sec:impl}
\subsection{Key Components}
\begin{description}
  \item[\texttt{EngramConfig}] Dataclass holding all hyperparameters
    (Table~\ref{tab:hparam}).
  \item[\texttt{EngramAdapter}] Core module: hash tables, projection, and gate.
  \item[\texttt{EngramInjectorForTraining}] Attaches adapters to the model at
    layer indices $\mathcal{L} = \{1,3,7,15\}$ via PyTorch forward hooks.
  \item[\texttt{get\_full\_attention\_layer\_ids()}] Identifies layers using full
    attention (vs.\ sliding-window attention); used to validate that all
    selected layer indices are full-attention layers.
  \item[\texttt{collect\_engram\_state\_dict()}] Extracts only Engram parameters
    for saving, leaving base model weights untouched.
\end{description}
\subsection{Hook Mechanism}
Two hooks are registered for the injected layers:
\begin{enumerate}
  \item \textbf{Model pre-hook}: Captures \texttt{input\_ids} and maintains a
    sliding window of \texttt{history\_ids} for autoregressive generation,
    aligned to $\max(\mathcal{N} \cup \mathcal{N}_{\mathrm{prev}})$.
  \item \textbf{Attention pre-hook} (per layer $\ell \in \mathcal{L}$): Runs
    \texttt{EngramAdapter} on hidden states \emph{before} the self-attention
    computation, injecting the residual per \eqref{eq:residual}.
\end{enumerate}
\subsection{Detailed Experimental Setup}
\label{app:exp-setup}
\begin{table*}[t]
\centering
\caption{Experimental setup summary. LoRA and Engram Adapter report mean $\pm$
std across three seeds on 4B; non-LoRA non-Engram 4B baselines and all
reported 8B baselines use seed 7783. For Qwen3-4B, all non-MBPP evaluations
use zero-shot prompting and MBPP uses 3-shot prompting; Qwen3-8B evaluations
use zero-shot prompting.}
\label{tab:setup}
\small
\setlength{\tabcolsep}{3pt}
\begin{tabularx}{\textwidth}{@{}l l c c X X@{}}
\toprule
\textbf{Model} & \textbf{Train task} & \textbf{Train size}
& \textbf{Eval} & \textbf{OOD evaluation} & \textbf{Baselines} \\
\midrule
Qwen3-4B & AG-News & 1{,}000 & 0-shot; MBPP 3-shot
& ARC-C, MedMCQA, FLORES, MBPP
& LoRA $r\in\{32,64\}$ (3 seeds), DoRA, PiSSA, Adapter, $(IA)^3$ \\
Qwen3-4B & MedMCQA & 1{,}000 & 0-shot; MBPP 3-shot
& ARC-C, AG-News, FLORES, MBPP
& LoRA $r\in\{32,64\}$ (3 seeds), DoRA, PiSSA, Adapter, $(IA)^3$ \\
Qwen3-8B & AG-News & 10{,}000 & zero-shot
& MedMCQA, ARC-C, LegalBench
& LoRA $r=64$, LoRA lr=$10^{-5}$, LoRA lr=$10^{-6}$, TALR \\
\bottomrule
\end{tabularx}
\end{table*}
\paragraph{Baselines and seeds.}
For 4B experiments, LoRA \citep{hu2022lora} is evaluated with
$r \in \{32,64\}$, $\alpha_{\mathrm{LoRA}}=2r$, learning rate $10^{-4}$, and
three seeds $\{42,7783,114514\}$. The Engram Adapter is also evaluated with
the same three seeds $\{42,7783,114514\}$. DoRA \citep{liu2024dora}, PiSSA
\citep{meng2024pissa}, Adapter \citep{houlsby2019adapter}, and $(IA)^3$
\citep{liu2022ia3} use seed 7783. For 8B experiments, LoRA $r=64$ with
lr=$10^{-5}$, LoRA $r=64$ with lr=$10^{-6}$, and TALR LoRA $r=64$
\citep{lin2026sft} also use seed 7783; their implementation details are given in
Appendix~\ref{app:8b-baselines}.
\paragraph{Benchmarks and metrics.}
AG-News \citep{zhang2015agnews} is evaluated on the full 7{,}600-example test
set. MedMCQA \citep{pal2022medmcqa} is evaluated on the 4{,}183-example
validation set. ARC-Challenge contains 1{,}172 science reasoning questions;
FLORES-200 contains 1{,}012 translation examples and is scored with BLEU and
chrF; MBPP contains 500 Python programming problems, is evaluated with 3-shot
prompting, and is scored with pass@1. All other Qwen3-4B evaluations use
zero-shot prompting.
For Qwen3-8B, we additionally evaluate 12 LegalBench
\citep{guha2023legalbench} tasks spanning contract analysis, evidence rules,
and jurisdictional reasoning.
\paragraph{Prompt formatting and hardware.}
To keep Engram's local-pattern lookup aligned with the answer-bearing portion
of each classification input, we append a short task instruction at the end of
the corresponding prompts. MedMCQA prompts end with ``Choose the correct answer
from A, B, C, or D based on medical analysis.'', and AG-News prompts end with
``Based on the article, choose its category.'' Final training and evaluation
runs were conducted on a single NVIDIA L20 GPU. 
\subsection{Additional 8B Baseline Details}
\label{app:8b-baselines}
\paragraph{Low learning rate and TALR baselines.}
To test whether forgetting can be mitigated by modulating the \emph{magnitude}
of weight updates rather than their \emph{conditionality}, we include three
additional baselines. LoRA $r$=64 (lr=$10^{-5}$) reduces the learning rate by
one order of magnitude relative to the standard setting ($10^{-4}$) and provides
a near-matched in-domain comparison to Engram Adapter. LoRA $r$=64
(lr=$10^{-6}$) reduces the learning rate by two orders of magnitude, limiting
the scale of $\Delta W$ more aggressively without any structural change. TALR
(Token-Adaptive Loss Reweighting) \citep{lin2026sft} instead modifies the
training objective by reweighting token-level losses, reducing the influence of
hard tokens that are associated with stronger general-capability degradation.
All three methods remain always-on at inference time.
\subsection{Training Configuration}
\begin{table}[h]
\centering
\caption{Hyperparameter defaults.}
\label{tab:hparam}
\scriptsize
\setlength{\tabcolsep}{2pt}
\begin{tabularx}{\linewidth}{@{}>{\ttfamily\scriptsize}p{0.35\linewidth}l>{\raggedright\arraybackslash}X@{}}
\toprule
\textbf{Hyperparameter} & \textbf{Default} & \textbf{Description} \\
\midrule
inject\_layers      & $(1,3,7,15)$ & Layer indices with adapters \\
ngram\_orders       & $(2,3)$      & Suffix $n$-gram sizes \\
prev\_ngram\_orders & $(3,)$       & Previous-context sizes \\
num\_hashes         & $4$          & Hash functions per lookup \\
table\_size         & $90{,}007$   & Buckets per hash table \\
per\_hash\_dim      & $128$ / $256$ & Embedding dim per hash (4B / 8B) \\
mem\_init\_std      & $0.02$       & Std of $\mathcal{N}(0, \sigma^2)$ init \\
num\_heads          & $16$         & Attention heads \\
gate\_type          & scalar       & Gating granularity \\
alpha               & $0.2$        & Residual strength \\
lr                  & $10^{-3}$    & Learning rate \\
grad\_accum\_steps  & $8$          & Gradient accumulation \\
max\_grad\_norm     & $1.0$        & Gradient clipping threshold \\
\bottomrule
\end{tabularx}
\end{table}
\section{Idealized Selectivity-Prior Details}
\label{sec:theory}
\subsection{Preliminaries and Notation}
Let $\model$ denote a frozen base LM and $\adapter$ an Engram Adapter. The
adapter modifies the hidden states at layers $\ell \in \mathcal{L} =
\{1,3,7,15\}$ per \eqref{eq:residual}. Let $\Dtrain$ denote the domain
training distribution and $\Dood$ any distribution disjoint from it. We define
the \emph{forgetting} induced by the adapter for input $x \sim \Dood$ as:
\[
  \Forgetting(x) = \bigl\|\model(x) - (\model + \adapter)(x)\bigr\|_2.
\]
Our goal is to separate two claims. First, before training, the explicit
occupancy mask makes the adapter exactly transparent. Second, after training,
the occupancy-tracked mask provides an idealized selectivity prior under
independent hashing, while practical OOD preservation must be established
empirically through residual and output-distribution analyses.
\subsection{Zero-Training Transparency}
\begin{proposition}
\label{prop:identity}
At initialization, before any training-domain $n$-gram has been recorded in
the occupancy masks, the Engram Adapter induces zero residual at every injected
layer. Consequently, for any input $x$, the adapted model is functionally
identical to the frozen base model and $\Forgetting(x)=0$.
\end{proposition}
\begin{proof}
All occupancy flags are initialized to $F_{n,h}[i]=0$ for all $n,h,i$.
Therefore, for any position $t$ and any hash table,
\[
\mu_{n,h}(t)=F_{n,h}[\mathrm{idx}_{n,h}(t)] = 0 .
\]
Since the joint mask for each order group is defined as the product over all
hash-specific occupancy indicators, we have
\[
m_n(t)=\prod_{h=0}^{k-1}\mu_{n,h}(t)=0 .
\]
Thus every retrieved chunk is masked out:
\[
\tilde{v}_{n,h}(t)=m_n(t)\,v_{n,h}(t)=\mathbf{0},
\]
regardless of the values stored in the randomly initialized embedding tables
$E_{n,h}$. The concatenated memory vector is therefore
$\bm{c}(t)=\mathbf{0}$, which gives
\[
\mathrm{mem}(t)=W_{\mathrm{proj}}\bm{c}(t)=\mathbf{0}.
\]
The residual injected at any adapted layer $\ell\in\mathcal{L}$ is
\[
\Delta u^{(\ell)}(t)=\alpha\cdot g^{(\ell)}(t)\cdot \mathrm{mem}^{(\ell)}(t)
=\mathbf{0},
\]
and hence $u_{\mathrm{new}}^{(\ell)}=u^{(\ell)}$. Layers
$\ell\notin\mathcal{L}$ are unmodified by construction. Therefore every hidden
state and output logit of the adapted model matches the frozen base model
exactly, so $\Forgetting(x)=0$ for all inputs $x$.
\end{proof}
\begin{remark}[Scope of the transparency guarantee]
\label{rem:zero_training_scope}
Proposition~\ref{prop:identity} is an exact initialization-time statement. It
does not claim that the adapter remains transparent after domain training. Once
training-domain $n$-grams have set occupancy flags to one, OOD inputs may
trigger false activations through hash collisions or shared local patterns.
Post-training OOD preservation is therefore not guaranteed by zero-training
transparency alone; it depends on the combination of mask selectivity, learned
gate attenuation, and projection-level residual suppression, which we analyze
empirically in Section~\ref{sec:activation}.
\end{remark}
\begin{remark}[Advantage over zero initialization]
\label{rem:filled_vs_zero}
An alternative design initializes all memory embeddings to $\mathbf{0}$ and
uses $\mu_{n,h}(t)=\1[\|v_{n,h}(t)\|_1>0]$ as the activation test. While this
can also make the adapter transparent before training, it entangles membership
tracking with embedding values. In particular, a bucket whose embedding remains
or becomes numerically close to zero may be treated as inactive, creating
spurious false negatives for previously observed training patterns. By contrast,
the explicit occupancy mask separates membership from representation learning:
$F_{n,h}$ records whether a bucket has been accessed by a training $n$-gram,
while $E_{n,h}$ stores the learned residual content. Since the occupancy flag is
monotone during training ($0\to1$ and never reset), exact repeats of recorded
training $n$-grams are not lost merely because their embedding values are small.
\end{remark}
\subsection{Idealized False Positive Analysis}
Let $\mathcal{N}_{\mathrm{train}}$ denote the set of unique $n$-grams
observed across all orders during domain training. After training, each such
$n$-gram has set $F_{n,h}[\mathrm{idx}_{n,h}] = 1$ in one bucket for each
hash table.
\begin{definition}[Fill rate]
The fill rate of hash table $(n,h)$ is the fraction of buckets with
$F_{n,h}[i] = 1$. Since each of the $|\mathcal{N}^{(n)}_{\mathrm{train}}|$
unique $n$-grams fills one bucket per hash, the expected fill rate is
$\rho_n = |\mathcal{N}^{(n)}_{\mathrm{train}}| / M$ (assuming no hash
collisions; with collisions the actual fill rate is lower, making this a
conservative fill-rate estimate).
\end{definition}
\begin{theorem}[Idealized false positive reference rate]
\label{thm:fpr}
Under uniform hashing and independence across the $k$ hash functions, an input
$n$-gram $q \notin \mathcal{N}^{(n)}_{\mathrm{train}}$ has joint-mask false
activation probability:
\[
  \begin{aligned}
  &P\bigl(m_n(t) = 1
    \;\big|\; q \notin \mathcal{N}^{(n)}_{\mathrm{train}}\bigr) \\
  &\quad \leq \rho_n^k .
  \end{aligned}
\]
This is an idealized reference rate for the stated assumptions, not a
distribution-free guarantee for natural language data.
\end{theorem}
\begin{proof}
A false positive at hash $h$ occurs if $\mathrm{idx}_{n,h}(q)$ maps to a
bucket that was filled by some training $n$-gram, i.e.,
$F_{n,h}[\mathrm{idx}_{n,h}(q)] = 1$. Under the idealized assumption that
hash outputs are uniformly distributed over $[M]$:
\[
  \begin{aligned}
  &P\bigl(\mu_{n,h}(t) = 1
    \;\big|\; q \notin \mathcal{N}^{(n)}_{\mathrm{train}}\bigr) \\
  &\quad \leq
    \frac{|\mathcal{N}^{(n)}_{\mathrm{train}}|}{M}
    = \rho_n .
  \end{aligned}
\]
Under the independence assumption across the $k$ hash outputs, the joint-mask
probability factorizes as:
\[
  \begin{aligned}
  P(m_n(t) = 1)
  &= \prod_{h=0}^{k-1} P(\mu_{n,h}(t)=1) \\
  &\leq \rho_n^k .
  \end{aligned}
\]
This proves the stated idealized reference rate.
\end{proof}
\begin{corollary}
For $|\mathcal{N}^{(n)}_{\mathrm{train}}| = 10{,}000$, $M = 90{,}007$, and $k=4$:
\[
  \rho_n^k = \left(\frac{10{,}000}{90{,}007}\right)^{\!4} \approx 1.52 \times 10^{-4}.
\]
\end{corollary}
\begin{remark}[Independence assumption]
\label{rem:independence}
Theorem~\ref{thm:fpr} assumes that the $k$ hash functions produce independent
outputs. In practice, all hash functions share the same token-ID input space,
and natural language $n$-gram distributions are highly non-uniform. As we show
empirically in Section~\ref{sec:activation}, OOD inputs still produce
non-negligible activations, so $\rho^k$ should not be treated as a formal
bound. It serves only as an idealized reference rate that illustrates how
increasing $k$ can reduce false activations in the independent-hashing limit.
The learned scalar gate (Section~\ref{sec:gate}) provides the primary
practical defense against OOD perturbation.
\end{remark}
\subsection{Gate as the Primary Defense Layer}
\label{sec:gate}
As shown empirically in Section~\ref{sec:activation}, OOD inputs trigger
non-negligible joint activation, meaning the occupancy mask alone does not
suppress all OOD contributions. The learned scalar gate~\eqref{eq:gate} serves
as the \emph{primary} defense mechanism.
When a false-positive activation occurs, the retrieved embeddings encode
domain-specific patterns trained on unrelated $n$-grams. For $u \sim \Dood$,
the weighted dot-product coherence term satisfies:
\[
  \frac{(W_u \odot u(t))^\top \mathrm{mem}(t)}{\sqrt{D}} \approx 0
\]
in expectation, following the high-dimensional heuristic that $u$ and
$W_{\mathrm{proj}}\,\bm{c}(t)$ are approximately orthogonal
\citep{vershynin2018highdim}. This reduces
the explicit coherence contribution from the gate; the learned linear term
$W_m\bm{c}(t)$ and the projected memory $\mathrm{mem}(t)$ then determine the
actual false-positive residual. The projected memory itself carries incoherent
signal that tends to cancel across the $k \times N$ retrieved chunks. The
combination of reduced gate values and incoherent
projections explains why the Engram Adapter achieves $\geq$99.4\% OOD retention
despite non-negligible activation rates: even when the mask admits a false
positive, the gate and projection layers attenuate its impact to a level
that is small in our residual and output measurements. Ablation experiments
(Section~\ref{sec:activation}) further show that the gate provides a
$\approx$1.8$\times$ suppression factor, with the learned linear terms
$W_m\bm{c}(t)$ accounting for nearly all of this effect.
\subsection{Idealized OOD Residual Bound}
\label{sec:idealized_bound}
The independence analysis above does not imply that the adapter is
post-training transparent on OOD inputs. After training, OOD examples may still
activate the occupancy mask through hash collisions or shared local patterns.
We therefore use the independent-hashing assumption only to derive an idealized
bound on the expected \emph{adapter residual}, rather than a distribution-free
guarantee on final model outputs.
Let the residual injected at position $t$ and layer $\ell$ be
\[
  \Delta u^{(\ell)}(t)
  =
  \alpha \cdot g^{(\ell)}(t) \cdot \mathrm{mem}^{(\ell)}(t),
\]
and define the per-example OOD residual proxy
\[
  \mathcal{R}(x)
  =
  \frac{1}{T}
  \sum_{t=1}^{T}
  \sum_{\ell \in \mathcal{L}}
  \|\Delta u^{(\ell)}(t)\|_2 .
\]
This proxy measures the total magnitude of the adapter intervention before it
propagates through the remaining transformer computation.
\begin{proposition}[Idealized OOD residual bound]
\label{prop:combined}
Consider an OOD input $x \sim \Dood$ whose queried $n$-grams are not exact
members of the training $n$-gram set. Suppose that the $k$ hash outputs are
independent and that each order group has fill rate at most $\rho$. For an
adapter injected at $|\mathcal{L}|$ layers with $N$ order groups, we have
\[
  \E\bigl[\mathcal{R}(x)\bigr]
  \;\leq\;
  |\mathcal{L}| \cdot \alpha \cdot N \cdot \rho^k
  \cdot \|\mathrm{mem}_{\mathrm{fp}}\|_{\infty},
\]
where $\|\mathrm{mem}_{\mathrm{fp}}\|_{\infty}$ denotes an upper bound on the
norm of the projected memory vector produced by a false-positive activation.
\end{proposition}
\begin{proof}
For a fixed token position and order group, Theorem~\ref{thm:fpr} gives
\[
  \Pr[m_n(t)=1] \leq \rho^k
\]
under independent hashing. By the union bound over $N$ order groups,
letting $A_t$ denote activation of any order group at position $t$,
\[
  \Pr(A_t)
  \leq N\rho^k .
\]
If no order group activates, all filtered chunks are zero, so
$\mathrm{mem}^{(\ell)}(t)=\mathbf{0}$ and
$\Delta u^{(\ell)}(t)=\mathbf{0}$. If a false-positive activation occurs, then
by definition
\[
  \|\mathrm{mem}^{(\ell)}(t)\|_2
  \leq
  \|\mathrm{mem}_{\mathrm{fp}}\|_{\infty}.
\]
Since $g^{(\ell)}(t)\in[0,1]$, the injected residual at one layer is bounded by
\[
  \|\Delta u^{(\ell)}(t)\|_2
  \leq
  \alpha \cdot \|\mathrm{mem}_{\mathrm{fp}}\|_{\infty}.
\]
Taking expectation over the false-positive activation event and summing over
$|\mathcal{L}|$ injected layers gives
\[
  \E[\mathcal{R}(x)]
  \leq
  |\mathcal{L}| \cdot \alpha \cdot N \cdot \rho^k
  \cdot \|\mathrm{mem}_{\mathrm{fp}}\|_{\infty}.
\]
\end{proof}
For our default configuration, $|\mathcal{L}|=4$, $N=3$, $k=4$,
$\rho \approx 0.222$, and $\alpha=0.2$, yielding the idealized estimate
\[
  |\mathcal{L}| \cdot \alpha \cdot N \cdot \rho^k
  \approx
  5.8 \times 10^{-3}.
\]
Thus, under independent hashing, the expected OOD residual proxy is bounded by
approximately
\[
  5.8 \times 10^{-3}
  \cdot
  \|\mathrm{mem}_{\mathrm{fp}}\|_{\infty}.
\]
\begin{remark}[Scope of the bound]
\label{rem:bound_scope}
Proposition~\ref{prop:combined} is an idealized residual-level analysis, not a
strict guarantee on $\Forgetting(x)$ or downstream accuracy. Bounding final
model outputs would require additional assumptions on the Lipschitz behavior of
the remaining transformer layers, which would be loose and difficult to verify
for large language models. We therefore use this proposition only to motivate
the role of the occupancy-tracked joint mask. In practice, the independence
assumption is violated by correlated hash inputs and non-uniform natural
language distributions, so empirical OOD activation can exceed $\rho^k$.
Post-training OOD preservation is consequently established by the residual and
output-distribution analyses in Section~\ref{sec:activation}, where the learned
gate and projection attenuation suppress the actual OOD residual to a negligible
fraction of the hidden-state norm.
\end{remark}
\subsection{Formal Comparison with LoRA}
Table~\ref{tab:comparison} summarizes the structural differences between LoRA
and the Engram Adapter.
\begin{table}[h]
\centering
\caption{Structural comparison of LoRA and Engram Adapter.}
\label{tab:comparison}
\small
\setlength{\tabcolsep}{3pt}
\begin{tabularx}{\linewidth}{@{}p{0.25\linewidth}p{0.27\linewidth}X@{}}
\toprule
\textbf{Property} & \textbf{LoRA} & \textbf{Engram Adapter} \\
\midrule
Adapted path & Always-on $\Delta W=BA$ & Additive residual over frozen $\theta$ \\
Activation condition & Always on & Occupancy-conditioned $n$-gram matches \\
OOD behavior & Full $\Delta W$ applied & Gate-attenuated ($<$0.1\% residual) \\
Retention mechanism & None (empirical only) & Gate + mask (empirical) \\
False positive control & N/A & Selectivity prior tunable via $k$ and $M$ \\
Composability & Requires merging / switching & Separate adapter module over frozen backbone \\
\bottomrule
\end{tabularx}
\end{table}
\begin{remark}
Without additional regularization or routing, LoRA exposes
$\Delta W = BA$ to every input. The Engram Adapter's mask provides a
stronger selectivity prior under the independent-hashing idealization, but real
language distributions require empirical validation through residual and output
analyses.
\end{remark}
\section{Full Retention and Additional Results}\label{app:full-results}
\subsection{Analysis}
\label{sec:analysis}
\paragraph{Retention preservation.}
The Engram Adapter achieves strong and stable retention across all OOD
benchmarks in both 4B experiments. In the AG-News experiment, average OOD
retention is 99.38\%$\pm$0.32\% across ARC-Challenge
(100.52\%$\pm$0.53\%), MedMCQA (99.84\%$\pm$0.38\%), FLORES-200 BLEU
(99.44\%$\pm$1.12\%), and MBPP (97.70\%$\pm$0.75\%). In the MedMCQA
experiment, the Engram Adapter achieves average retention of
100.07\%$\pm$0.24\% across ARC-Challenge (99.42\%$\pm$0.38\%), AG-News
(101.37\%$\pm$0.62\%), FLORES-200 BLEU (100.13\%$\pm$0.18\%), and MBPP
(99.35\%$\pm$0.38\%). These results are consistent with the
selectivity-prior picture, but the retention claim itself is established by
the residual and output analyses in Section~\ref{sec:activation}: the
adapter's OOD perturbation is negligible for practical purposes.
On 8B, the retention pattern is consistent: MedMCQA and ARC-Challenge retention
both exceed 99.7\%, and LegalBench average retention is 100.9\%, yielding
101.0\% average OOD retention across the three 8B OOD benchmarks. The
cross-scale consistency suggests that the same selectivity-plus-attenuation
mechanism extends across model sizes.
In contrast, LoRA exhibits consistent MBPP degradation across all six MedMCQA
configurations on 4B, with pass@1 retention ranging from 95.21\% (LoRA $r$=64,
seed 42) to 97.60\% (LoRA $r$=32, seeds 42 and 114514). In the AG-News
experiment, forgetting is even more severe: LoRA $r$=64 with seed 42 drops
MBPP from 66.80\% to 39.40\%---a 41\% relative degradation. On 8B, LoRA's
forgetting is most pronounced on LegalBench, where average accuracy drops by
27.2 absolute points. The intermediate learning rate variant (lr=$10^{-5}$)
nearly matches Engram Adapter in-domain (86.96\% vs.\ 86.47\%) but reaches only
96.0\% average OOD retention, driven primarily by a LegalBench drop to 61.08\%.
The more conservative low learning rate variant (lr=$10^{-6}$) achieves
97.2\% average OOD retention, with MedMCQA slightly below the base model
(53.60\% vs.\ 55.51\%). TALR achieves strong non-LegalBench OOD retention (101.9\%
across MedMCQA and ARC-Challenge, 98.2\% when LegalBench is included) but with
only marginal in-domain improvement (79.57\% vs.\ base 79.49\%). This
suggests that LoRA's always-on weight perturbation
$\Delta W = BA$ can substantially disrupt capabilities across diverse domains
and model scales, and that neither learning rate reduction nor token-adaptive
loss reweighting simultaneously achieves both near-matched adaptation and
reliable retention in this low-resource LoRA setting.
\paragraph{Other PEFT methods can exhibit severe forgetting.}
PiSSA exhibits the most extreme OOD degradation among all baselines despite
not always failing in-domain: when trained on MedMCQA, it reaches 55.63\%
in-domain accuracy but drops AG-News to 0.00\% and ARC-Challenge to 27.30\%;
when trained on AG-News, it reduces MedMCQA to 32.23\% and ARC-Challenge to
19.88\%. We note that PiSSA was run with the same learning rate as LoRA
($10^{-4}$), which may be suboptimal for PiSSA's principal-subspace
reparameterization; a dedicated hyperparameter sweep could yield substantially
better results and we therefore treat its collapse as a configuration artifact
rather than a fundamental limitation of the method. $(IA)^3$ trained on MedMCQA degrades ARC-Challenge from 88.99\% to
62.88\% (29.3\% relative loss). DoRA achieves competitive in-domain
performance (90.30\% on AG-News, 57.28\% on MedMCQA) and preserves reasoning
ability (ARC-Challenge 88.14\%/88.31\%), but Adapter, while strong on
in-domain metrics, also degrades OOD benchmarks. These results underscore that
unconditional parameter modification across multiple PEFT method families can
lead to severe and unpredictable capability degradation, even with only 1{,}000
training examples.
\paragraph{Seed sensitivity of LoRA.}
The MedMCQA experiment, with three seeds per rank, provides a more robust
assessment of LoRA's seed sensitivity. At $r$=32, in-domain accuracy varies
from 56.71\% to 57.30\% across seeds---a modest 0.59-point range. However,
OOD metrics also fluctuate: ARC-Challenge ranges from 87.96\% to 88.31\%, and
MBPP from 65.00\% to 65.20\%. At $r$=64, the variance is similar: in-domain
accuracy ranges from 54.15\% to 54.84\%, while MBPP ranges from 63.60\% to
64.40\%. In the AG-News experiment, seed sensitivity is far more dramatic: at
$r$=64, MBPP retention ranges from 58.98\% to 97.31\%---a 38-point gap. The
Engram Adapter's OOD behavior is substantially more stable across seeds: in
the AG-News experiment, MBPP retention ranges from 97.01\% to 98.50\% (a
1.5-point range), and in the MedMCQA experiment from 98.95\% to 99.70\% (a
0.8-point range). This stability is attributable to the occupancy mask
constraining activation and the gate/residual path attenuating perturbations.
\paragraph{Rank--forgetting tradeoff in LoRA.}
Reducing LoRA's rank from 64 to 32 consistently mitigates forgetting. In the
MedMCQA experiment, the average OOD retention across three seeds improves from
99.10\%--99.56\% ($r$=64) to 99.64\%--99.77\% ($r$=32). In the AG-News
experiment, average retention similarly improves from 88.31\%--99.51\% ($r$=64)
to 98.66\%--99.81\% ($r$=32). The Engram Adapter (99.38\%$\pm$0.32\% on
AG-News, 100.07\%$\pm$0.24\% on MedMCQA) achieves the highest average
retention in both experiments---tying with LoRA $r$=32 on AG-News
(99.38\%$\pm$0.51\%) while exhibiting lower variance, and leading on MedMCQA.
Notably, unlike LoRA $r$=32 whose AG-News average retention is partly inflated
by MedMCQA positive transfer (102.32\%), the Engram Adapter's retention is
more uniformly
distributed across all OOD benchmarks, with no single benchmark exhibiting
severe degradation. The Engram Adapter achieves consistently high retention
across both experiments without requiring rank selection or seed tuning.
\paragraph{Domain adaptation effectiveness.}
In the 4B AG-News experiment, the Engram Adapter achieves
87.93\%$\pm$0.64\% accuracy, improving over the base model (85.14\%) by
2.79 percentage points but lagging the best adapted model (Adapter, 90.54\%)
by 2.61 points. On 8B, the in-domain improvement
is substantially larger: 86.47\% vs.\ base 79.49\% (+6.98pp), narrowing the
relative gap. In the MedMCQA experiment, the
Engram Adapter achieves 55.79\%$\pm$0.74\% accuracy (+2.14 over base), while
the best LoRA achieves 57.30\% ($r$=32, seed 114514, +3.65 over base). The
in-domain gap relative to the best LoRA reflects the fundamental
specialization--preservation tradeoff: the conditional activation mechanism
limits the adapter's influence to recognized $n$-gram patterns.
\paragraph{Cross-domain consistency.}
A key finding is that the Engram Adapter's retention behavior holds
consistently across both model scales (4B and 8B), both domain adaptation tasks,
and all OOD benchmarks,
covering reasoning (ARC-Challenge), text classification (AG-News),
translation (FLORES-200), code generation (MBPP), and legal reasoning
(LegalBench). This consistency supports the selectivity-prior interpretation:
the occupancy-tracked matching mechanism provides a domain-agnostic inductive bias that is
not specific to the training domain or model scale, while the
residual/output analyses explain why false activations do not translate into
behavioral forgetting.

\subsection{Retention Summary}
Tables~\ref{tab:retention_agnews} and \ref{tab:retention_medmcqa} report the
average OOD retention for each 4B experiment.
\begin{table*}[t]
\centering
\caption{Average OOD retention (\%) for AG-News domain adaptation (4B).
LoRA and Engram Adapter results are mean $\pm$ std across three seeds.
Higher is better; 100\% = perfect preservation.}
\label{tab:retention_agnews}
\small
\begin{tabular}{@{}lccccc@{}}
\toprule
\textbf{Method} & \textbf{ARC-C} & \textbf{MedMCQA} & \textbf{FLORES (BLEU)} & \textbf{MBPP} &
  \textbf{Avg.\ Ret.} \\
\midrule
LoRA $r\!=\!64$
  & \underline{99.04{\tiny$\pm$0.55}} & 100.12{\tiny$\pm$0.62} & 97.01{\tiny$\pm$1.48}
  & 84.83{\tiny$\pm$18.28} & 95.25{\tiny$\pm$4.95} \\
LoRA $r\!=\!32$
  & 99.01{\tiny$\pm$0.43} & \textbf{102.32{\tiny$\pm$0.50}} & 98.98{\tiny$\pm$0.62}
  & \underline{97.21{\tiny$\pm$1.25}} & \textbf{99.38{\tiny$\pm$0.51}} \\
Adapter
  & 98.75 & 99.55 & \underline{100.20} & 83.83 & 95.58 \\
DoRA $r\!=\!32$
  & \underline{99.04} & \underline{101.73} & \textbf{100.77} & 89.52 & \underline{97.77} \\
PiSSA $r\!=\!32$
  & 22.34 & 60.07 & 0.00 & 0.00 & 20.60 \\
$(IA)^3$
  & 96.65 & 99.27 & 71.09 & 84.43 & 87.87 \\
\midrule
\textbf{Engram Adapter}
  & \textbf{100.52{\tiny$\pm$0.53}} & 99.84{\tiny$\pm$0.38} & 99.44{\tiny$\pm$1.12}
  & \textbf{97.70{\tiny$\pm$0.75}} & \textbf{99.38{\tiny$\pm$0.32}} \\
\bottomrule
\end{tabular}
\end{table*}
\begin{table*}[t]
\centering
\caption{Average OOD retention (\%) for MedMCQA domain adaptation (4B).
LoRA and Engram Adapter results are mean $\pm$ std across three seeds.
Higher is better; 100\% = perfect preservation.}
\label{tab:retention_medmcqa}
\small
\begin{tabular}{@{}lccccc@{}}
\toprule
\textbf{Method} & \textbf{ARC-C} & \textbf{AG-News} & \textbf{FLORES (BLEU)} &
  \textbf{MBPP} & \textbf{Avg.\ Ret.} \\
\midrule
LoRA $r\!=\!64$
  & 99.14{\tiny$\pm$0.21} & \underline{101.36{\tiny$\pm$0.28}} & \underline{100.69{\tiny$\pm$0.96}}
  & 95.91{\tiny$\pm$0.51} & 99.28{\tiny$\pm$0.20} \\
LoRA $r\!=\!32$
  & 99.04{\tiny$\pm$0.16} & 101.16{\tiny$\pm$0.14} & \textbf{101.06{\tiny$\pm$0.34}}
  & \underline{97.50{\tiny$\pm$0.14}} & \underline{99.69{\tiny$\pm$0.06}} \\
Adapter
  & 98.85 & 99.38 & 97.92 & 84.43 & 95.15 \\
DoRA $r\!=\!32$
  & \underline{99.23} & 99.55 & 100.23 & 85.33 & 96.09 \\
PiSSA $r\!=\!32$
  & 30.68 & 0.00 & 0.00 & 0.00 & 7.67 \\
$(IA)^3$
  & 70.66 & 99.34 & 70.10 & 86.83 & 81.74 \\
\midrule
\textbf{Engram Adapter}
  & \textbf{99.42{\tiny$\pm$0.38}} & \textbf{101.37{\tiny$\pm$0.62}} & 100.13{\tiny$\pm$0.18} & \textbf{99.35{\tiny$\pm$0.38}} & \textbf{100.07{\tiny$\pm$0.24}} \\
\bottomrule
\end{tabular}
\end{table*}
In the AG-News experiment, the Engram Adapter and LoRA $r$=32 tie for the
highest average retention at 99.38\% (Engram $\pm$0.32\%, LoRA $\pm$0.51\%),
with the Engram Adapter exhibiting lower variance. However, LoRA $r$=32's
high average is partly driven by
MedMCQA positive transfer (102.32\%), while its MBPP retention (97.21\%) is
lower than the Engram Adapter's (97.70\%). The remaining methods trail
further: DoRA (97.77\%), LoRA $r$=64 (95.25\%$\pm$4.95\%), Adapter (95.58\%),
$(IA)^3$ (87.87\%), and PiSSA (20.60\%). The OOD collapse of PiSSA and the
FLORES degradation of $(IA)^3$ (71.09\% retention) demonstrate that
not all PEFT methods degrade gracefully---some suffer substantial capability
loss on specific benchmarks.
In the MedMCQA experiment, the Engram Adapter achieves the highest average
retention at 100.07\%$\pm$0.24\%, followed by LoRA $r$=32
(99.69\%$\pm$0.06\%), LoRA $r$=64 (99.28\%$\pm$0.20\%), DoRA (96.09\%),
Adapter (95.15\%), $(IA)^3$ (81.74\%), and PiSSA (7.67\%).
\section{Full LegalBench Results}\label{app:legalbench}
\paragraph{LegalBench evaluation.}
To test OOD preservation on a challenging legal domain, we evaluate all
8B models (trained on AG-News) on 12 LegalBench tasks spanning contract
analysis, evidence rules, and jurisdictional reasoning.
Table~\ref{tab:legalbench} presents the results.
\begin{table*}[t]
\centering
\caption{LegalBench results on Qwen3-8B (trained on AG-News, evaluated OOD on
12 legal reasoning tasks). Values are task accuracies (\%).}
\label{tab:legalbench}
\footnotesize
\setlength{\tabcolsep}{3pt}
\begin{tabular}{@{}l cccccc@{}}
\toprule
\textbf{Task} & \textbf{Base} & \textbf{Engram} & \textbf{LoRA64} & \textbf{lr=$10^{-5}$} & \textbf{lr=$10^{-6}$} & \textbf{TALR} \\
\midrule
abercrombie                   & 36.84 & 40.00 & 36.84 & 23.16 & 37.89 & 32.63 \\
consumer\_contracts\_qa       & 90.40 & 90.91 & 69.70 & 92.68 & 92.17 & 92.68 \\
contract\_nli\_confid.        & 78.05 & 78.05 & 74.39 & 79.27 & 78.05 & 89.02 \\
contract\_nli\_limited        & 82.69 & 83.17 & 65.87 & 81.25 & 83.17 & 81.25 \\
contract\_nli\_return         & 81.82 & 81.82 & 75.76 & 83.33 & 81.82 & 87.88 \\
contract\_qa                  & 97.50 & 97.50 & 97.50 & 97.50 & 97.50 & 98.75 \\
corporate\_lobbying           & 61.02 & 65.10 & 22.04 & 38.98 & 62.86 & 29.80 \\
function\_of\_decision        & 89.65 & 94.55 & 0.00 & 86.65 & 67.57 & 73.02 \\
hearsay                       & 56.38 & 56.38 & 31.91 & 55.32 & 56.38 & 35.11 \\
personal\_jurisdiction        & 62.00 & 60.00 & 4.00 & 48.00 & 62.00 & 42.00 \\
proa                          & 64.21 & 65.26 & 44.21 & 6.32 & 60.00 & 51.58 \\
unfair\_tos                   & 62.39 & 57.67 & 13.77 & 40.44 & 43.11 & 68.14 \\
\midrule
\textbf{Average}              & 71.91 & 72.53 & 44.67 & 61.08 & 68.54 & 65.16 \\
\bottomrule
\end{tabular}
\end{table*}
The Engram Adapter achieves 72.53\% average accuracy across 12 LegalBench
tasks, slightly exceeding the base model (71.91\%) and preserving performance
on 10 out of 12 tasks. The only degradations are relatively small in absolute size
(personal\_jurisdiction $-$2.0pp, unfair\_tos $-$4.7pp). Since the 8B
experiments are single-seed, future work should quantify LegalBench variance
with additional seeds or bootstrap confidence intervals.
Several tasks show noticeable improvement, including
function\_of\_decision (+4.9pp) and corporate\_lobbying (+4.1pp), suggesting that
the AG-News classification training may provide mild positive transfer to
structurally similar legal classification tasks.
We also run LegalBench-only component ablations on the Engram Adapter. Removing
the occupancy mask (\textbf{Engram-no-mask}) yields 72.06\% average accuracy,
and disabling the scalar gate (\textbf{Engram-no-gate}) yields 71.86\%. These
ablations are close to the frozen base model but below the full Engram Adapter,
suggesting that LegalBench average accuracy is relatively insensitive to either
single component while still favoring the complete mask-plus-gate design.
In contrast, LoRA $r$=64 degrades on 10 out of 12 tasks, with
large drops on function\_of\_decision (89.65\% $\to$ 0.00\%) and
personal\_jurisdiction (62.00\% $\to$ 4.00\%). The average drops from 71.91\%
to 44.67\%---a 37.9\% relative degradation.
LoRA (lr=$10^{-5}$) provides a near-matched in-domain comparison to Engram
Adapter on AG-News (86.96\% vs.\ 86.47\%), but its LegalBench average is only
61.08\%. It degrades on 8 out of 12 tasks, with especially large drops on proa
($-$57.9pp), unfair\_tos ($-$22.0pp), and corporate\_lobbying ($-$22.0pp).
LoRA (lr=$10^{-6}$) achieves 68.54\% average accuracy, substantially better
than standard LoRA (44.67\%) and LoRA lr=$10^{-5}$ (61.08\%), and second only
to the Engram Adapter among adapted models. The conservative learning rate
preserves performance on 9 out of 12 tasks, but produces large degradation on
function\_of\_decision
($-$22.1pp) and unfair\_tos ($-$19.3pp), with a smaller drop on proa
($-$4.2pp). This pattern reveals a characteristic weakness of
magnitude-reduction approaches: while a low learning rate limits the overall
scale of $\Delta W$, specific capability subspaces can still be disrupted
disproportionately, and the practitioner has no control over \emph{which}
capabilities are affected.
TALR LoRA achieves 65.16\% average
accuracy, outperforming standard LoRA (44.67\%) but falling
short of LoRA (lr=$10^{-6}$) (68.54\%), the base model (71.91\%), and the
Engram Adapter (72.53\%). It also outperforms the near-matched LoRA
lr=$10^{-5}$ baseline on LegalBench (65.16\% vs.\ 61.08\%), but only by
sacrificing in-domain learning. TALR shows
a mixed pattern: it improves on 5 tasks (notably contract\_nli\_confidentiality
+11.0pp and unfair\_tos +5.8pp) but degrades on 7 tasks, with the largest
drops on corporate\_lobbying ($-$31.2pp), hearsay ($-$21.3pp), and
personal\_jurisdiction ($-$20.0pp). This pattern suggests that token-adaptive
loss reweighting mitigates but does not eliminate the forgetting caused by
unconditional weight modification.

\section{Additional Activation and Ablation Details}\label{app:activation-ablation}
\begin{figure*}[t]
\centering
\includegraphics[width=0.95\textwidth]{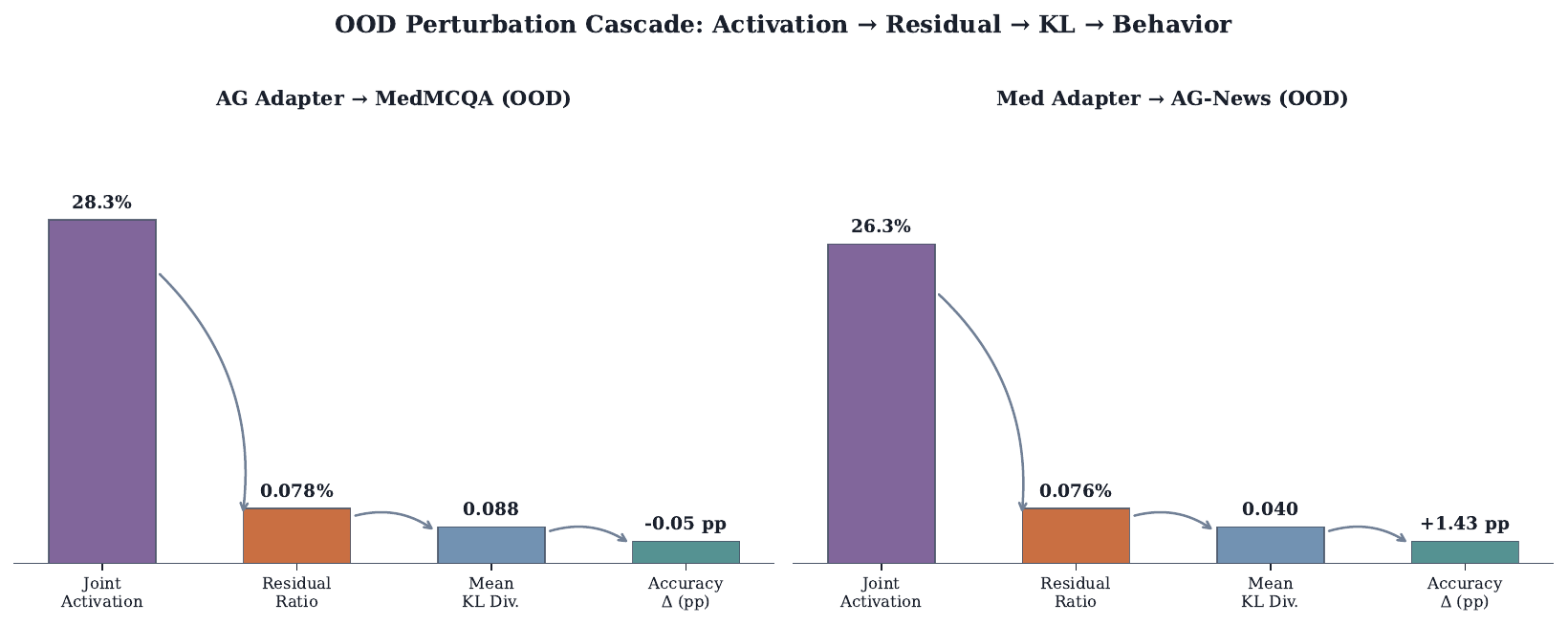}
\caption{OOD perturbation cascade on the two highest-activation cross-domain
pairs. Despite $\approx$27\% token-level activation, the gate suppresses
residual perturbation to $<$0.08\% of hidden-state norm, yielding minimal
KL divergence and negligible accuracy change.}
\label{fig:cascade}
\end{figure*}
\subsection{Matched-Coverage Gate Selectivity Diagnostic}
We further evaluate whether different lightweight or retrieval-based signals can
separate held-out target-domain prompts from OOD prompts. This is a gate
diagnostic rather than a downstream classification-accuracy comparison. For
each scoring rule, we set the threshold to the 5th percentile of held-out
AG-News scores, so that approximately 95\% of held-out AG-News examples are
accepted without using OOD data for threshold selection. We then report the OOD
false-accept rate, i.e., the fraction of OOD examples whose score is at least
this threshold.

\paragraph{Engram $n$-gram density.}
This is the native Engram selectivity signal. After registering occupied hash
buckets from the AG-News training examples, each prompt is scored by the
fraction of valid prompt tokens for which at least one local $n$-gram group
(n2, n3, or p3) activates all four occupied hash entries.

\paragraph{TF-IDF max cosine.}
This lexical retrieval baseline represents each prompt with sparse TF-IDF
features and scores an evaluation prompt by its maximum cosine similarity to
the AG-News training examples. It tests whether simple surface-form retrieval
is sufficient as a domain gate.

\paragraph{Qwen embedding retrieval.}
This dense retrieval baseline embeds each prompt with the frozen Qwen3-4B
backbone and scores an evaluation prompt by its maximum cosine similarity to
the AG-News training embeddings. It provides a stronger semantic retrieval
signal, but requires an additional dense encoding and retrieval path that is
separate from Engram's local hash-memory lookup.

\begin{table}[t]
\centering
\caption{Matched-coverage OOD false-accept rate (\%) for domain-gate signals.
All thresholds are selected at the 5th percentile of held-out AG-News scores,
corresponding to approximately 95\% in-domain acceptance. Lower is better.}
\label{tab:gate_selectivity}
\footnotesize
\setlength{\tabcolsep}{2pt}
\begin{tabularx}{\linewidth}{@{}>{\raggedright\arraybackslash}Xccc@{}}
\toprule
\textbf{Signal} & \textbf{Med.} & \textbf{ARC-C} & \textbf{Legal} \\
\midrule
Engram $n$-gram density & 0.60 & 0.17 & 0.27 \\
TF-IDF max cosine & 16.90 & 24.06 & 31.53 \\
Qwen embedding retrieval & \textbf{0.00} & \textbf{0.00} & \textbf{0.10} \\
\bottomrule
\end{tabularx}
\end{table}
\subsection{Gate Ablation}
To isolate the contribution of the gate, we ablate it on the same two OOD
pairs, measuring the effect on the residual perturbation ratio among activated
tokens. We test two gate variants:
\begin{itemize}
  \item \textbf{$g=1$ (gate disabled)}: Sets $g(t) = 1$ for all positions,
    removing the gate's suppression entirely.
  \item \textbf{no\_dot}: Removes the dot-product coherence term
    $(W_u \odot u(t))^\top \mathrm{mem}(t) / \sqrt{D}$ from the gate,
    retaining only the learned linear term $W_m\bm{c}(t)$.
\end{itemize}
Table~\ref{tab:ablation} reports the results.
\begin{table*}[t]
\centering
\caption{Gate ablation on the two highest-activation OOD pairs. Metrics are
mean and P95 residual ratio
($\|\Delta u\| / \|u\|$, \%) among activated tokens. $\uparrow$ indicates
degradation (higher OOD perturbation); $\downarrow$ indicates improvement.}
\label{tab:ablation}
\small
\begin{tabular}{@{}ll cc cc@{}}
\toprule
& & \multicolumn{2}{c}{\textbf{AG $\to$ MedMCQA}} &
     \multicolumn{2}{c}{\textbf{Med $\to$ AG-News}} \\
\cmidrule(lr){3-4} \cmidrule(lr){5-6}
\textbf{Type} & \textbf{Variant}
  & Mean\% & P95\% & Mean\% & P95\% \\
\midrule
--- & Baseline & 0.275 & 1.124 & 0.290 & 1.233 \\
\midrule
\multirow{2}{*}{Gate}
  & $g=1$ (disabled) & 0.505$\uparrow$ & 2.124$\uparrow$ & 0.509$\uparrow$ & 1.990$\uparrow$ \\
  & no\_dot          & 0.273 & 1.086 & 0.286 & 1.105 \\
\bottomrule
\end{tabular}
\end{table*}
Disabling the gate ($g=1$) increases the mean residual ratio by
1.8$\times$ (AG $\to$ Med) and 1.8$\times$ (Med $\to$ AG), confirming that
the gate provides substantial residual-level attenuation. However, removing the dot-product
coherence term (\textbf{no\_dot}) has virtually no effect: the mean residual ratios are nearly unchanged, while the P95 residuals
decrease slightly. This means the gate's effectiveness
comes primarily from the \emph{learned linear term} $W_m\bm{c}(t)$, which has
learned to suppress incoherent retrievals, rather than from the explicit
coherence signal
$(W_u \odot u(t))^\top \mathrm{mem}(t) / \sqrt{D}$.
The LegalBench-only no-gate result in Appendix~\ref{app:legalbench} further
shows that this larger residual does not necessarily translate into large
average-accuracy loss on every OOD benchmark.

\section{Inference Efficiency and Activation Profile}\label{app:efficiency}
 
To characterize the runtime cost of the Engram Adapter, we report two
complementary profiling views. First, Table~\ref{tab:controlled_efficiency}
reports supplementary controlled measurements collected during the rebuttal
period on a single NVIDIA RTX 6000D using Qwen3-4B. The training profile uses
one epoch over 1k AG-News examples with batch size 8, cutoff length 1024, bf16,
and AdamW. The inference profile fixes batching and maximum input length 256;
generation uses fixed 32-token greedy decoding. These controlled inference
numbers measure Engram Adapter overhead relative to the frozen base model;
decode-step overhead is $+9.18\%$--$+10.99\%$.

\begin{table*}[t]
\centering
\caption{Controlled supplementary efficiency profile on Qwen3-4B
AG-News. Training numbers use an AG-News-1k one-epoch profile on a single
NVIDIA RTX 6000D with batch size 8, cutoff length 1024, bf16, and AdamW.}
\label{tab:controlled_efficiency}
\footnotesize
\setlength{\tabcolsep}{2.3pt}
\begin{tabular}{@{}lccc@{}}
\toprule
\textbf{Method} & \textbf{Params} & \textbf{Train profile} &
\textbf{Peak mem.} \\
\midrule
Base Model
  & 0 & --- & --- \\
LoRA $r\!=\!64$
  & 132.1M & 20.88s / 47.88 ex/s & 23.2GB \\
LoRA $r\!=\!32$
  & 66.1M & 19.85s / 50.38 ex/s & 22.4GB \\
\textbf{Engram Adapter}
  & 154.0M & 19.80s / 50.50 ex/s & 17.5GB \\
\bottomrule
\end{tabular}
\end{table*}

Although Engram Adapter has more trainable parameters than both LoRA baselines,
its measured peak training memory is lower in this profiling setting. We do not
interpret this as a general memory-dominance claim: in our implementation, LoRA
is attached to more projection modules, requiring additional module-input
activations to be retained for backpropagation, which may partly explain the
higher measured peak memory.

Second, Table~\ref{tab:efficiency} reports the original benchmark-dependent
profile from the final evaluation environment. These runs were conducted on a
single NVIDIA L20 GPU with benchmark-dependent sequence lengths and evaluation
windows. They are useful for per-dataset diagnostics and for relating latency
to activation rates, but they are not directly comparable to the controlled
RTX 6000D measurements in Table~\ref{tab:controlled_efficiency}.
Forward overhead is measured as the relative increase in mean prefill latency
compared with the frozen base model on the same inputs. Generation overhead is
measured on MBPP with a fixed maximum of 128 generated tokens for both the base
and adapted models; for the remaining benchmarks, generation overhead is not
reported because the Engram Adapter's domain-adapted outputs differ in length
from the base model's, making total generation time not directly comparable.
 
\begin{table*}[t]
\centering
\caption{Original benchmark-dependent inference overhead and $n$-gram
activation profile on Qwen3-4B.
Forward overhead measures mean prefill latency increase; generation
overhead is reported only for MBPP (fixed 128-token output).
Activation columns report the fraction of token positions where the
joint mask fires for each $n$-gram channel.
For in-domain rows (shaded), generation overhead is not reported
because output-length differences confound the comparison.
FLORES-200 forward overhead is omitted due to high variance across
adapters ($\mathrm{SD}=25.5\%$, $n=2$). These measurements were collected in
the final L20 evaluation environment and should not be directly compared with
the controlled RTX 6000D profile in Table~\ref{tab:controlled_efficiency}.}
\label{tab:efficiency}
\small
\setlength{\tabcolsep}{3pt}
\begin{tabular}{@{}ll cc ccc@{}}
\toprule
& & \multicolumn{2}{c}{\textbf{Latency Overhead}}
& \multicolumn{3}{c}{\textbf{Activation Rate}} \\
\cmidrule(lr){3-4} \cmidrule(lr){5-7}
\textbf{Adapter} & \textbf{Eval Set}
  & Forward & Generation & n2 & n3 & p3 \\
\midrule
\rowcolor{gray!10}
AG-News & AG-News (in-domain)
  & +12.3\% & --- & 0.500 & 0.356 & 0.376 \\
AG-News & MedMCQA
  & +39.9\% & --- & 0.219 & 0.126 & 0.137 \\
AG-News & ARC-C
  & +34.5\% & --- & 0.194 & 0.052 & 0.062 \\
AG-News & MBPP
  & +1.45\% & +21.8\% & 0.397 & 0.221 & 0.221 \\
\midrule
\rowcolor{gray!10}
MedMCQA & MedMCQA (in-domain)
  & +36.6\% & --- & 0.717 & 0.601 & 0.610 \\
MedMCQA & AG-News
  & +6.5\% & --- & 0.198 & 0.096 & 0.106 \\
MedMCQA & ARC-C
  & +34.5\% & --- & 0.313 & 0.145 & 0.158 \\
MedMCQA & MBPP
  & +1.56\% & +22.7\% & 0.347 & 0.122 & 0.126 \\
\bottomrule
\end{tabular}
\end{table*}
 
\paragraph{Latency and activation profile.}
Forward overhead ranges from $+1.45\%$--$+1.56\%$ on MBPP to
$+39.9\%$ on AG$\to$MedMCQA. However, the measured latency is not
monotonic in the token-level activation rate. For example, MBPP shows
higher average activation than AG$\to$MedMCQA for the AG-News adapter,
but much lower prefill overhead. This indicates that runtime is affected
not only by adapter activation, but also by input length distribution,
batching, Python-level hash and lookup overhead, memory-access patterns,
and measurement variance. We therefore use Table~\ref{tab:efficiency}
primarily to report the practical overhead of the current implementation,
while interpreting the activation columns as evidence of selectivity
rather than as a direct latency predictor.
 
\paragraph{In-domain vs.\ OOD activation contrast.}
In-domain activation rates are substantially higher than OOD rates
across all channels. For the AG-News adapter, in-domain n3 activation
is 0.356 versus 0.052--0.221 on OOD sets; for the MedMCQA adapter,
in-domain n3 activation is 0.601 versus 0.096--0.145 on OOD sets.
This gradient from near-domain to far-domain inputs provides direct
evidence for the selectivity prior discussed in
Section~\ref{sec:activation}.
 
\paragraph{Engineering optimization.}
The current implementation uses unoptimized Python-level hash
computation and occupancy lookup. The hash and gather operations are
amenable to fused GPU kernels, which we expect would substantially
reduce the observed overhead. Inference optimization is left to future
work.

\section{Additional Discussion}\label{app:discussion}
\label{sec:discussion}
\paragraph{Why we do not report gate-only LoRA as a formal baseline.}
A natural control is to add a learned scalar gate to an otherwise standard
LoRA residual. We do not report this variant as a formal baseline because its
gate is randomly initialized and trained only on the target-domain adaptation
data, with no calibrated supervision for unrelated OOD domains. In our
implementation, three Gated-LoRA trials with learning rates $10^{-4}$,
$10^{-5}$, and $10^{-6}$ all produced empty outputs on LegalBench, resulting in
degenerate evaluations rather than a meaningful accuracy comparison. We therefore
treat this as an instability diagnosis, not as a formal baseline result.
The relevant distinction is architectural: Engram provides a discrete,
occupancy-based membership signal before residual injection, whereas a
gate-only LoRA variant must infer when to suppress an always-available LoRA
residual from continuous hidden-state statistics alone.
\paragraph{Why training-time retention controls are insufficient.}
The contrasting failures of LoRA (lr=$10^{-5}$), LoRA (lr=$10^{-6}$), and TALR
illuminate a fundamental limitation of training-time retention controls.
Reducing the global learning rate uniformly shrinks $\|\Delta W\|$ across all
layers and all input patterns, yet the resulting perturbation---however small---
is still applied unconditionally to every input. At lr=$10^{-5}$, LoRA nearly
matches Engram Adapter in-domain on AG-News (86.96\% vs.\ 86.47\%) but drops to
61.08\% average accuracy on LegalBench. At lr=$10^{-6}$, the perturbation is
smaller and LegalBench improves to 68.54\%, but the AG-News gain shrinks and
MedMCQA remains below the base model (53.60\% vs.\ 55.51\%). This suggests that
specific capability subspaces can still be disrupted by non-selective updates
even when the update magnitude is reduced.
TALR takes a different training-time approach by reweighting token-level losses,
and indeed achieves strong OOD retention on the standard benchmarks (101.9\%).
However, this comes at a cost in our low-resource LoRA setting:
down-weighting hard-token losses also suppresses useful task-learning signals,
resulting in only marginal in-domain improvement (79.57\% vs.\ base 79.49\%).
On LegalBench, TALR outperforms standard LoRA (65.16\% vs.\ 44.67\%) and
near-matched LoRA lr=$10^{-5}$ (61.08\%) but falls short of both the more
conservative low learning rate variant (68.54\%) and the base model (71.91\%),
suggesting that neither uniform magnitude reduction nor token-adaptive loss
reweighting can match conditional activation's combination of effective
adaptation and reliable retention. These results underscore the
distinction between \emph{modulating how the model is trained} and
\emph{controlling when} the adaptation is applied at inference time. Among the
tested methods, Engram Adapter achieves effective adaptation and reliable
retention by addressing the \emph{conditionality} of intervention rather than only the
training objective or update magnitude.
\paragraph{Conservation vs.\ positive transfer.}
A nuanced consequence of the Engram Adapter's design is that OOD performance
is largely preserved rather than enhanced. We observe that LoRA models trained
on AG-News achieve slightly improved accuracy on MedMCQA (up to 54.90\% on average,
compared to the 53.65\% base model), despite never seeing medical data. This
effect is attributable to improved instruction-following ability: the AG-News
classification training implicitly teaches the model to follow task formatting
conventions (e.g., selecting among labeled options), which transfers to
structurally similar tasks like MedMCQA's multiple-choice format. In contrast,
the Engram Adapter trained on AG-News achieves 53.57\%$\pm$0.20\% on
MedMCQA---close to the base model---indicating that although the adapter does
activate on some OOD tokens (Section~\ref{sec:activation}), the gate mechanism
effectively suppresses the contribution to a level that neither substantially
degrades nor enhances OOD performance.
Thus, our method should be understood as a
retention-oriented specialization mechanism rather than a way to harvest
incidental global instruction-following gains from domain fine-tuning.
Interestingly, on 8B, the Engram Adapter shows mild positive transfer on some
LegalBench tasks (e.g., function\_of\_decision +4.9pp), suggesting that at
larger scale, the adapter's domain-discriminative $n$-grams may partially
overlap with structurally similar legal classification patterns.
\paragraph{Future work.}
We plan to extend this work along several axes: (1) scaling experiments on
models up to 27B parameters with domain corpora spanning law
(MultiLegalPile) and finance (FinPile); (2) systematic ablation of $k \in
\{1,2,4,6,8\}$ and $M \in \{45{,}007;\, 90{,}007;\, 300{,}007\}$ to measure how empirical false positives deviate from the
idealized reference rate in Theorem~\ref{thm:fpr}; (3) multi-domain adapter-isolation experiments where
independent adapters are trained for medicine, law, and finance, then evaluated
on all three domains to confirm cross-domain isolation; (4) training data
scaling experiments (1k--50k examples) to characterize how the in-domain gap
narrows with corpus size; and (5) inference latency and memory profiling across
methods.
\end{document}